\documentclass{article}
\PassOptionsToPackage{table}{xcolor}

\usepackage[main, final]{neurips_2026}

\usepackage[utf8]{inputenc}
\usepackage[T1]{fontenc}
\usepackage{hyperref}
\usepackage{url}
\usepackage{booktabs}
\usepackage{amsfonts}
\usepackage{nicefrac}
\usepackage{microtype}
\usepackage{subcaption}
\usepackage{dsfont}
\usepackage{xparse}
\usepackage{amsmath, amssymb}
\usepackage{xcolor}
\usepackage{listings}
\definecolor{mintbg}{HTML}{f8f8f8}
\definecolor{mintframe}{HTML}{cccccc}
\lstdefinestyle{mintedlike}{
  backgroundcolor=\color{mintbg},
  basicstyle=\ttfamily,
  columns=flexible,
  keepspaces=true,
  breaklines=true,
  breakatwhitespace=false,
  frame=single,
  rulecolor=\color{mintframe},
  framesep=3mm,
  xleftmargin=20pt,
  xrightmargin=20pt,
  framexleftmargin=10pt,
  tabsize=4,
  showstringspaces=false,
  postbreak=\mbox{\textcolor{gray}{$\hookrightarrow$}\space}
}
\usepackage{upquote}
\usepackage[most]{tcolorbox}
\definecolor{promptph}{HTML}{27509B}
\lstdefinestyle{promptplain}{
  basicstyle=\ttfamily\footnotesize,
  columns=fixed,
  basewidth=0.5em,
  upquote=true,
  breaklines=true,
  breakatwhitespace=true,
  breakautoindent=true,
  breakindent=0pt,
  showstringspaces=false,
  moredelim=[s][\color{promptph}]{<}{>},
  escapeinside={(*}{*)}
}
\newtcblisting{promptbox}[1]{
  enhanced, breakable, listing only,
  colback=mintbg, colframe=mintframe, boxrule=0.4pt, arc=1.5pt,
  colbacktitle=mintframe!45, coltitle=black, fonttitle=\small\bfseries, title={#1},
  left=6pt, right=6pt, top=3pt, bottom=3pt, toptitle=2pt, bottomtitle=2pt,
  before skip=8pt, after skip=10pt,
  listing options={style=promptplain}
}
\newcommand{\promptrole}[1]{\textcolor{black!55}{\textsf{\bfseries\scriptsize\MakeUppercase{#1}}}}
\newcommand{\promptlit}[1]{\textcolor{black}{\texttt{\textless#1\textgreater}}}
\usepackage{tikz}
\usetikzlibrary{tikzmark,calc}
\usepackage{makecell}
\usepackage{multirow}

\usepackage[flushmargin,hang]{footmisc}
\hypersetup{
  colorlinks=true,
  linkcolor=blue!35!black,
  citecolor=blue!35!black,
  urlcolor=blue!35!black
}

\usepackage{inconsolata}

\usepackage[greek.ancient,english]{babel}

\usepackage{hyperref}
\usepackage{url}
\usepackage{booktabs}
\usepackage{amsfonts}
\usepackage{nicefrac}
\usepackage{microtype}

\usepackage{latexsym}
\usepackage{amsfonts}
\usepackage{amsmath}
\usepackage{amssymb}
\usepackage{amsthm}

\usepackage{thmtools}
\usepackage{thm-restate}

\usepackage{relsize}
\usepackage{mathtools}
\usepackage{bm}
\usepackage{dsfont}
\usepackage{xfrac}
\usepackage{stmaryrd}
\SetSymbolFont{stmry}{bold}{U}{stmry}{m}{n}
\usepackage[makeroom]{cancel}

\usepackage{xspace}
\usepackage{adjustbox}
\usepackage{relsize}

\usepackage{url}

\usepackage{array}

\usepackage[inline,shortlabels]{enumitem}
\setitemize{noitemsep,topsep=0pt,parsep=0pt,partopsep=0pt,leftmargin=*}
\setenumerate{noitemsep,topsep=0pt,parsep=0pt,partopsep=0pt,leftmargin=*}

\usepackage{caption}

\usepackage[normalem]{ulem}

\newtheorem{definition}{Definition}[section]

\newtheorem{proposition}{Proposition}[section]
\theoremstyle{definition}
\newtheorem{example}{Example}[section]

\AtBeginEnvironment{example}{\pushQED{\hfill //}}
\AtEndEnvironment{example}{\popQED\endexample}

\usepackage{cleveref}
\crefname{section}{\S}{\S\S}
\crefname{table}{Tab.}{Tabs.}
\crefname{figure}{Fig.}{Figs.}
\crefname{algorithm}{Alg.}{Algs.}
\crefname{equation}{Eq.}{Eqs.}
\crefname{example}{Ex.}{Exs.}
\crefname{fact}{Fact}{Facts}
\crefname{appendix}{App.}{Appendices}
\crefname{theorem}{Thm.}{Thms.}
\crefname{reTheorem}{Thm.}{Thms.}
\crefname{aquestion}{Question}{Questions}
\crefname{assumption}{Assumption}{Assumptions}
\crefname{lemma}{Lem.}{Lemmas}
\crefname{reLemma}{Lem.}{Lemmas}
\crefname{proposition}{Prop.}{Props.}
\crefname{chapter}{Chapter}{Chapters}
\crefname{line}{line}{lines}
\crefname{principle}{Principle}{Principles}
\crefname{definition}{Def.}{Defs.}
\crefname{corollary}{Cor.}{Cors.}
\crefname{Exercise}{Exercise}{Exercises}
\crefformat{section}{\S#2#1#3}
\crefformat{footnote}{#2Footnote~#1#3}

\usepackage{colortbl}

\definecolor{ETHBlue}{RGB}{33,92,175}
\definecolor{ETHGreen}{RGB}{98,115,19}
\definecolor{ETHPurple}{RGB}{163,7,116}
\definecolor{ETHGray}{RGB}{111,111,111}
\definecolor{ETHRed}{RGB}{183,53,45}
\definecolor{ETHPetrol}{RGB}{0,120,148}
\definecolor{ETHBronze}{RGB}{142,103,19}

\colorlet{ETHdarkblue}{ETHBlue!80!black}
\colorlet{ETHdarkgreen}{ETHGreen!80!black}
\colorlet{ETHpink}{ETHPurple}
\colorlet{ETHgray}{ETHGray}
\colorlet{ETHred}{ETHRed}
\colorlet{ETHgreenblue}{ETHPetrol}
\colorlet{ETHbrown}{ETHBronze}

\colorlet{MacroColor}{ETHPetrol}
\colorlet{MACROCOLOR}{MacroColor}

\usepackage{amssymb}

\usepackage{todonotes}
\makeatletter
\newcommand*\iftodonotes{\if@todonotes@disabled\expandafter\@secondoftwo\else\expandafter\@firstoftwo\fi}
\makeatother

\definecolor{mintgreen}{RGB}{152, 255, 152}
\definecolor{DiverseMagenta}{RGB}{192,0,128}

\usepackage{algorithm}
\usepackage[noend]{algpseudocode}
\algrenewcommand\algorithmicindent{1.0em}%

\newcommand{\rightcomment}[1]{{\color{gray} \(\triangleright\){\footnotesize\textit{#1}}}}
\algrenewcommand{\algorithmiccomment}[1]{\hfill \rightcomment{#1}}
\algnewcommand{\LinesComment}[1]{\State {\color{black!50!green}\rightcomment{\parbox[t]{.95\linewidth-\leftmargin-\widthof{\(\triangleright\) }}{#1}}}}
\algnewcommand{\LineComment}[1]{\State {\color{black!50!green}\smaller \(\triangleright\) \parbox[t]{\linewidth-\leftmargin-\widthof{\(\triangleright\) }}{\it #1}\smallskip}}
\algnewcommand{\InlineComment}[1]{\hfill {\color{black!50!green}\(\triangleright\) {\scriptsize \it #1}}}
\algrenewcommand\algorithmicindent{1.0em}%

\algrenewcommand\alglinenumber[1]{{\tiny\color{black!50}#1.}\hspace{-2pt}}
\newcommand{\algorithmicfunc}[1]{\textbf{def} {#1}:}
\algdef{SE}[FUNC]{Func}{EndFunc}[1]{\algorithmicfunc{#1}}{}
\makeatletter
\ifthenelse{\equal{\ALG@noend}{t}}%
{\algtext*{EndFunc}}
{}%
\makeatother

\usepackage{multicol}

\colorlet{MacroColor}{black}
\newcommand{\mymacro}[1]{{\color{MacroColor} #1}}

\newcommand{\defn}[1]{\textbf{#1}}
\newcommand{\defeq}{\mathrel{\stackrel{\textnormal{\tiny def}}{=}}}

\newcommand{\simplexFun}[1]{\mymacro{\triangle}^{#1}}

\newlength\myheight
\newlength\mydepth
\settototalheight\myheight{Xygp}
\newcommand*\inlinegraphicsSmall[1]{%
  \settototalheight\myheight{Xygp}%
  \settodepth\mydepth{Xygp}%
  \raisebox{0pt}[0pt][-10pt]{\includegraphics[height=0.8\myheight]{#1}}%
}

\newcommand{\kleene}[1]{{\mymacro{#1^\ast}}}
\newcommand{\alphabet}{{\mymacro{\Gamma}}}
\newcommand{\str}{{\mymacro{\boldsymbol{\gamma}}}}
\newcommand{\sym}{{\mymacro{\gamma}}}
\newcommand{\emptystr}{{\mymacro{\varepsilon}}}

\newcommand{\Rnonneg}[0]{\mymacro{\mathbb{R}_{\ge 0}}}
\newcommand{\RnonnegK}[0]{\mymacro{\mathbb{R}^{\K}_{\ge 0}}}
\newcommand{\K}[0]{\mymacro{K}}

\newcommand{\func}[0]{\mymacro{\ensuremath{f}}\xspace}
\newcommand{\ensemble}[0]{\ensuremath{\mymacro{\Phi{}}}\xspace}
\newcommand{\targetFn}[0]{\ensuremath{\mymacro{\varphi}\xspace}}

\newcommand{\targetFnPrefix}[0]{\ensuremath{\mymacro{\vec{\targetFn}}\xspace}}
\newcommand{\ensemblePrefix}[0]{\ensuremath{\mymacro{\vec{\ensemble}}\xspace}}

\newcommand{\p}{\ensuremath{\mymacro{p}}\xspace}
\newcommand{\q}{\ensuremath{\mymacro{q}}\xspace}
\newcommand{\strings}{\mymacro{\alphabet^*}}

\newcommand{\xx}{\ensuremath{\mymacro{\boldsymbol{x}}}\xspace}

\renewcommand{\Pr}[2]{\mathop{\mathbb{P}}_{\substack{#1}}\left[#2\right]}
\newcommand{\E}[0]{\mathop{\mathbb{E}}}
\newcommand{\KL}[0]{\mathrm{KL}}

\DeclareMathOperator*{\argmin}{\mymacro{argmin}}

\newcommand{\simIID}[0]{\overset{\text{\tiny{i.i.d.}}}{\sim}}

\newcommand{\proposal}[0]{{\color{black}r}}
\newcommand{\proposalPrefix}[0]{{\color{black}\vec{r}}}

\newcommand{\shapingPrefix}[0]{{\color{black}\vec{\psi}}}

\newcommand{\familyParam}{\mymacro{\tau}}
\newcommand{\essThreshold}{\mymacro{\eta}}
\newcommand{\pExpert}{\mymacro{p}}
\newcommand{\ptokens}{\mymacro{p_{\tokAlphabet}}}
\newcommand{\pchars}{\mymacro{p_{\charAlphabet}}}

\newcommand{\decodingFun}{\mymacro{\kappa}}
\newcommand{\tokenizer}{\mymacro{\zeta}}
\makeatletter
\newcommand{\pEns}{\@ifnextchar\bgroup\pEns@with\pEns@bare}
\newcommand{\pEns@bare}{\mymacro{\p_{\mathrm{ens}}}}
\newcommand{\pEns@with}[1]{\mymacro{\p_{\mathrm{ens}, #1}}}
\makeatother

\newcommand{\tokAlphabet}{{\mymacro\Delta}}
\newcommand{\Y}{\mymacro{Y}}

\newcommand{\mychar}[0]{x}
\newcommand{\chars}{{\mymacro{\mychar}}}
\newcommand{\charseq}{{\mymacro{\boldsymbol{\mychar}}}}
\newcommand{\charAlphabet}{{\mymacro{\Sigma}}}
\newcommand{\charStrings}[0]{{\mymacro{\kleene{\charAlphabet}}}}

\newcommand{\indicator}[1]{\mathds{1}\{ #1 \}}

\newcommand{\prefixOp}[1]{\mymacro{\vec{{\normalcolor#1}}}}
\newcommand{\pPrefix}[0]{\prefixOp{\p}}

\newcommand{\eos}[0]{\mymacro{\textup{\textsc{eos}}}}

\newcommand{\proofExplain}[1]{\text{{\color{black!50}[#1]}}}

\newcommand{\activeParticle}[0]{\mymacro{\alpha}}
\newcommand{\totalWeight}[0]{\mymacro{W}}
\newcommand{\importanceWeight}[0]{\mymacro{w}}
\newcommand{\Zest}[0]{\mymacro{\widehat{Z}}}
\newcommand{\ess}[0]{\mymacro{\widehat{M}}}
\newcommand{\tmpWeight}[0]{\mymacro{\overline{w}}}

\renewcommand{\hat}[1]{\widehat{#1}}
\newcommand{\expertPriorWeight}{\mymacro{w}}
\newcommand{\expertPriorWeightVec}{\mymacro{\boldsymbol w}}

\usepackage{ellipsis}
\makeatletter
\renewcommand*{\mathellipsis}{%
  \mathinner{%
    \kern\ellipsisbeforegap%
    {\ldotp}\kern\ellipsisgap%
    {\ldotp}\kern\ellipsisgap%
    {\ldotp}\kern\ellipsisaftergap%
  }%
}
\renewcommand*{\dotsb@}{%
  \mathinner{%
    \kern\ellipsisbeforegap%
    {\cdotp}\kern\ellipsisgap%
    {\cdotp}\kern\ellipsisgap%
    {\cdotp}\kern\ellipsisaftergap%
  }%
}
\renewcommand*{\@cdots}{%
  \mathinner{%
    \kern\ellipsisbeforegap%
    {\cdotp}\kern\ellipsisgap%
    {\cdotp}\kern\ellipsisgap%
    {\cdotp}\kern\ellipsisaftergap%
  }%
}
\renewcommand*{\ellipsis@default}{%
  \ellipsis@before
  \kern\ellipsisbeforegap
  .\kern\ellipsisgap
  .\kern\ellipsisgap
  .\kern\ellipsisgap
  \ellipsis@after\relax}
\renewcommand*{\ellipsis@centered}{%
  \ellipsis@before
  \kern\ellipsisbeforegap
  .\kern\ellipsisgap
  .\kern\ellipsisgap
  .\kern\ellipsisaftergap
  \ellipsis@after\relax}
\AtBeginDocument{%
  \DeclareRobustCommand*{\dots}{%
    \ifmmode\@xp\mdots@\else\@xp\textellipsis\fi}}
\def\ellipsisgap{-.05em}
\def\ellipsisbeforegap{-.05em}
\def\ellipsisaftergap{.05em}
\makeatother

\newcommand{\res}[2]{$#1_{\pm #2}$}

\newcommand{\win}[1]{\textbf{\boldmath #1}}
\newcommand{\sig}{\textsuperscript{†}}

\renewcommand{\bar}[1]{\overline{#1}}
\renewcommand{\hat}[1]{\widehat{#1}}

\usepackage{booktabs}
\usepackage[table]{xcolor}
\usepackage{arydshln}
\newcommand{\greydashedrule}{%
  \arrayrulecolor{gray!80}\hdashline[1pt/2pt]\arrayrulecolor{black}%
}
\newcommand{\myparagraph}[1]{\textbf{#1}\quad}
\title{Ensembling Language Models \\ with Sequential Monte Carlo}

\author{%
  Robin Shing Moon Chan$^{1,*}$%
  \quad
  Tianyu Liu$^{1}$%
  \quad
  Samuel Kiegeland$^{1,2}$%
  \quad
  Clemente Pasti$^{1,2}$%
  \\
  \textbf{Jacob Hoover Vigly}$^{2}$%
  \quad
  \textbf{Timothy J.~O'Donnell}$^{2,3,4,5}$%
  \quad
  \textbf{Ryan Cotterell}$^{1,*}$%
  \quad
  \textbf{Tim Vieira}$^{1,*}$%
  \\
  \mdseries
  $^{1}$ETH Z\"urich \quad
  $^{2}$CHI-FRO \quad
  $^{3}$McGill University \quad
  $^{4}$Mila \quad
  $^{5}$Canada CIFAR AI Chair
}

\begin{document}

\maketitle

\begingroup
\renewcommand{\thefootnote}{\fnsymbol{footnote}}
\footnotetext[0]{$^{*}$Corresponding authors: \texttt{\{\href{mailto:robin.chan@inf.ethz.ch}{robin.chan}, \href{mailto:ryan.cotterell@inf.ethz.ch}{ryan.cotterell}\}@inf.ethz.ch}, \texttt{\href{mailto:tim.f.vieira@gmail.com}{tim.f.vieira@gmail.com}}}
\endgroup

\begin{abstract}
Combining the predictions of multiple models is one of the most effective strategies in machine learning.
Thus, it is only natural to investigate the ensembling of language models (LMs).
However, because LMs define distributions over \emph{strings}, a proper ensemble requires aggregating predictions at the string level, which requires a computation that is generally intractable.
In this work, we cast LM ensembling as an inference problem and propose the deployment of sequential Monte Carlo (SMC), an appropriate approximate inference scheme.
Concretely, we introduce a unified framework for composing multiple LMs into ensemble distributions parameterized by a broad family of aggregation functions.
To sample from these distributions, we introduce an SMC algorithm that operates in a shared byte space, enabling ensembles of models with mismatching vocabularies and consistent sampling in the limit.
We evaluate a range of ensembles across prompt and model combinations for various text generation tasks, finding that \textit{consensus-seeking} strategies such as the product consistently outperform probability averaging, and that better posterior approximations can yield better ensemble performance.

\medskip
{\centering
\raisebox{-0.25em}{\includegraphics[height=1em]{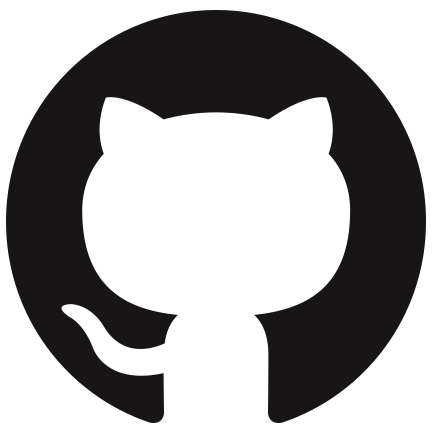}}
\url{https://github.com/genlm/ensemble-neurips-2026}\par}
\end{abstract}

\section{Introduction}

Ensembling methods have long been a cornerstone of machine learning: combining multiple predictors into a single, stronger one is one of the most reliable strategies for improving performance \cite{kittler1998combining}.
For instance, probability averaging reduces variance \cite{jordan1994mixtures}, the product of experts \cite{hinton1999products} concentrates probability mass on regions of agreement, and boosting \cite{freund1997decision} sequentially corrects errors.
We take the position that ensembling is more relevant than ever in the age of large language models (LMs) as practitioners increasingly rely on \emph{pre-trained} LMs whose parameters they typically cannot easily update.
Moreover, different LMs, or even different prompts applied to the same model, can surface complementary strengths, making their combination especially attractive.

However, ensembling LMs is not straightforward.
Because LMs define distributions over \emph{strings}, ensembling, as traditionally understood, means aggregating predictions at the string level.
Unfortunately, computing and sampling from such an ensemble distribution exactly is generally intractable, as the set of all strings is countably infinite.
Indeed, most approaches sidestep string-level aggregation entirely by combining next-token distributions during ancestral (autoregressive) sampling.\footnote{Even when simply aggregating next token probability, ensembling is nontrivial as models with different tokenizations cannot be combined directly.
Prior work has proposed various alignment heuristics, e.g., union vocabularies \citep{yu2024breaking, yao2025determinethenensemble} and shared embedding spaces \citep{huang2024ensemble, xu2024bridging}.
Once alignment is achieved, however, these methods predominantly default to (weighted) token probability averaging (cf.\ \Cref{sec:related-work}), which results in sampling from a biased approximation of the true ensemble distribution over strings \citep{park2024grammaraligned}.}

\begin{figure*}[t!]
\centering
\includegraphics[width=0.86\linewidth]{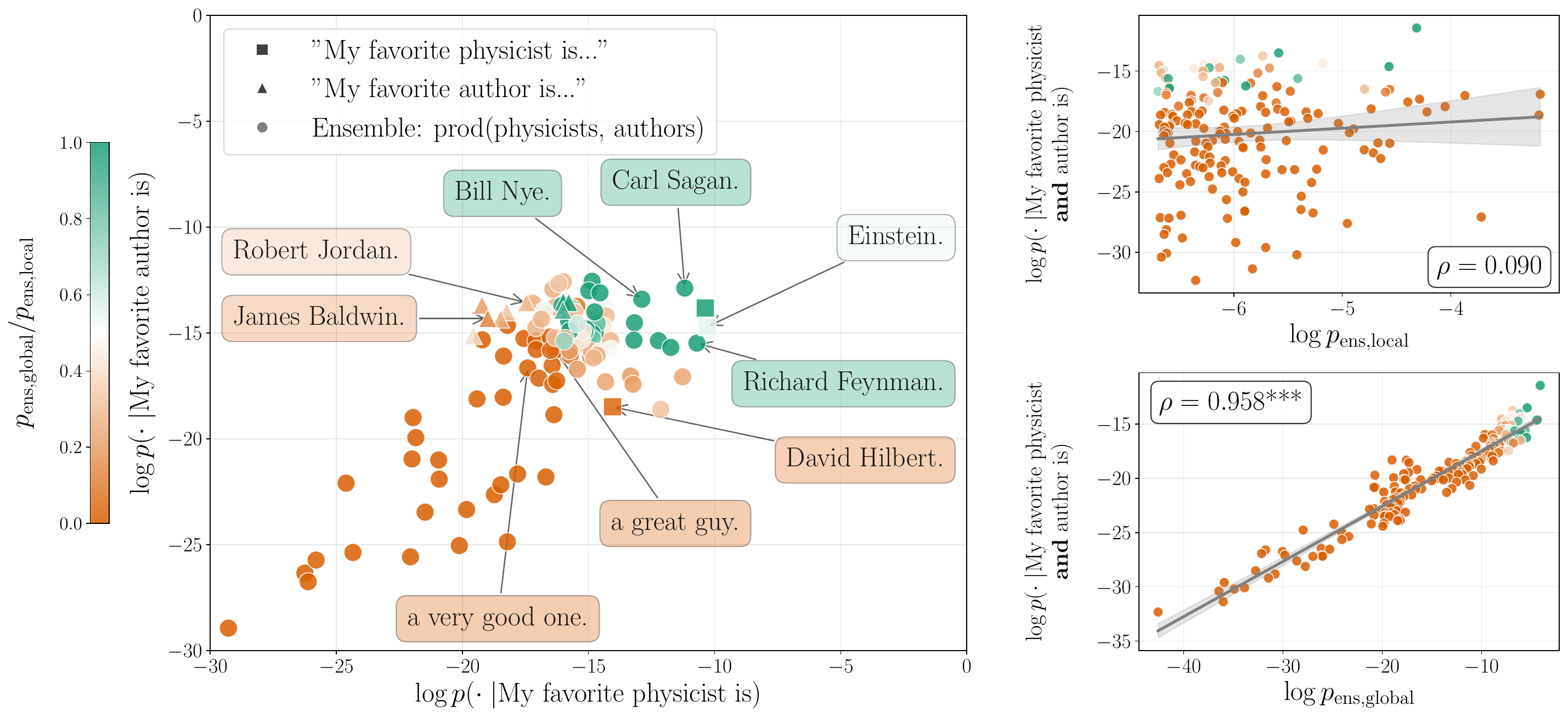}
\caption{Prompt intersection on GPT-2. We generate 200 strings from the normalized token-level product of ``\textit{My favorite physicist is}'' and ``\textit{My favorite author is}'' prompts, then score each under the local ensemble $p_{\mathrm{ens, local}}$ (string probabilities emerging from \textit{locally normalized} token-level product ensembling) and global ensemble $p_{\mathrm{ens, global}}$ (normalized product of string probabilities). Left: scatter plot of log probabilities with color indicating the generated strings' global probability relative to their local ones. We also plot completions from the author prompt (\protect\inlinegraphicsSmall{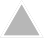}) and the physicist prompt (\protect\inlinegraphicsSmall{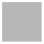}) that don't overlap with the ensemble's completions (\protect\inlinegraphicsSmall{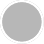}). Right: Correlation with the explicitly formulated intersection constraint probability $p(\cdot \mid \text{``\textit{My favorite physicist \textbf{and} author is}''})$.}
\label{fig:prompt-intersection}
\end{figure*}

In this work, we take a different stance: we cast ensembling as an \emph{inference} problem and tackle it with sequential Monte Carlo \cite[SMC;][]{doucet2001sequential,delmoral2004feynman,huggins2014information,chopin2020introduction}.
Crucially, SMC is a well-understood algorithm that converges to the true string-level ensemble distribution in the limit of an infinite number of particles.

Our formulation starts with the introduction of $\func$-ensembles: a unified framework for composing $\K$ LMs via an arbitrary aggregation function $\func\colon \RnonnegK \rightarrow \Rnonneg$, and a byte-level SMC algorithm that enables consistent sampling from any $\func$-ensemble in the limit, independently of their respective tokenization. 
Within this framework, we focus on the family of \textit{generalized means}, which, parameterized by a single parameter $\tau$, recovers many well-known aggregation strategies such as the mixture of experts \citep{jordan1994mixtures} at $\tau = 1$ and the product of experts \citep{hinton1999products} at $\tau \to 0$.
We show that generalized means arise naturally from a variational principle: they are the unique minimizers of weighted sums of $\alpha$-divergences \citep{amari1985differential} to the base LMs.%
This variational characterization lends interpretability to the choice of $\tau$: each value corresponds to optimizing a distinct divergence measure between distributions, determining how the ensemble mediates disagreement among the component LMs.
Concretely, $\tau \le 0$ yields \emph{consensus-seeking} behavior, concentrating mass on the intersection of the component LMs' supports, while $\tau > 0$ yields \emph{diversity-preserving} behavior, spreading mass across all component LMs.

Our main empirical contribution is to evaluate and compare these $f$-ensembles on a range of text generation tasks, including both within-model prompt ensembling and ensembles across model families.
We find that language models can act synergistically: $f$-ensembling consistently yields accuracy improvements over the best individual model.
Among ensembling strategies, the choice of $f$ has a significant effect on performance: consensus-seeking strategies, such as the product, consistently outperform prior works' standard approach of token-level probability averaging.
Finally, we show that better posterior approximations (obtained by increasing the number of SMC particles) can translate into better ensemble performance, confirming that principled inference over the global ensemble distribution can yield practical gains.

\section{Motivation: Global Prompt Intersection}

We motivate our approach with a toy experiment (cf. \cref{fig:prompt-intersection}) illustrating how the choice of aggregation function $f$ shapes the semantics of the ensemble, why string-level ensemble distributions differ from their token-level approximations, and why better posterior approximations matter in practice.
To this end, we adapt a \textit{prompt intersection} experiment from \citet{lew2023sequential}.\looseness=-1

\myparagraph{Setup.} Consider a language model prompted with two prefixes with overlapping plausible completions.
A practitioner may be interested in sampling strings that are probable under both prompts, i.e., sampling from their pointwise product.
For illustrative purposes, we prompt GPT-2 \citep{radford2019language} with ``\textit{My favorite physicist is}'' and ``\textit{My favorite author is}'' and aim to synthesize an ensemble that places high probability on physicists who are also authors. To that end, we generate 200 completions ending in a period using best-first beam search \citep{meister2020best}
(beam width 20), scoring partial sequences by the cumulative, locally normalized product of next-token probabilities.

\textbf{String Scoring.} We compare two ways of scoring \textit{completed} strings under the ensemble.
The \emph{local ensemble} $p_{\mathrm{ens, local}}$ multiplies next-token probabilities from each prompt at every generation step and normalizes at each step, yielding string probabilities as a product of locally normalized factors.
The \emph{global ensemble} $p_{\mathrm{ens, global}}$ instead computes each prompt's probability for the entire completed string and multiplies these string-level probabilities, normalizing across the 200 generated completions.\footnote{The true marginal likelihood is intractable, as it requires an infinite sum over all possible strings. However, these approximate probabilities only differ from the true ones by a constant.}

\myparagraph{Results.} In \autoref{fig:prompt-intersection}, we plot the 200 ensemble completions (\protect\inlinegraphicsSmall{figs/Circle.pdf}) alongside completions generated from each prompt individually (\protect\inlinegraphicsSmall{figs/Triangle.pdf} for authors, \protect\inlinegraphicsSmall{figs/Rectangle.pdf} for physicists).  
We see that the completions generated from the locally normalized product ensemble tend to be ones that are assigned high probability under both prompts, rather than those probable under only one prompt (e.g., ``\textit{James Baldwin.}'', a probable completion for the author prompt, is improbable under the physicist prompt, making it also low-probability under the ensemble). However, when comparing the probabilities assigned by the local ensemble with the global one, we observe a systematic mismatch: the local ensemble places relatively more probability mass on globally low-probability continuations, such as ``\textit{a great guy.}'', and relatively less probability on globally high-probability strings, such as ``\textit{Carl Sagan.}''. This occurs because the local product rewards tokens that are individually probable under both prompts at each step, e.g., favoring generic continuations like ``\textit{a}'' or ``\textit{my}'' early on rather than rewarding full completions that are probable under both prompts. A completion like ``\textit{a great guy.}'' accumulates high local scores despite being improbable under the string-level global ensemble.

We validate this observation by scoring the generated completions under a prompt that explicitly encodes the intersection constraint: ``\textit{My favorite physicist \textbf{and} author is}''. The global ensemble probabilities strongly correlate with this explicit constraint ($\rho = 0.958$, $p < 0.001$), while the local ensemble does not ($\rho = 0.090$, $p = 0.24$).

\myparagraph{Insights.} 
This example illustrates two key points that motivate our approach. First, global string-level ensembling can yield markedly different behavior than local token-level ensembling, and better posterior approximations more faithfully capture the intended ensemble distribution. Second, the product is just one aggregation strategy: different choices of $f$ induce qualitatively different behaviors over strings. Together, these observations raise the central question of this work: given an arbitrary aggregation function $f$, how can we efficiently sample from its global 
ensemble distribution?

\section{Background: Language Modeling}\label{sec:background}
An \defn{alphabet} $\alphabet$ is a finite, non-empty set of symbols. Let $\strings$ denote the set of all strings $\str$ formed from the symbols in $\alphabet$, including the empty string $\emptystr$. We write $|\str|$ to denote the length of $\str$.
A \defn{language model} $\p$ is a probability distribution over $\strings$. We write $\simplexFun{\strings}$ to denote the set of valid probability distributions over $\strings$.
Let $\str' \succeq \str$ denote that $\str$ is a prefix of $\str'$.
The \defn{prefix probability} $\pPrefix(\str)$ and \defn{conditional prefix probability} $\pPrefix(\str' \mid \str)$ are defined as

\begin{subequations}
\noindent
\begin{minipage}[c]{0.43\linewidth}
\begin{equation}
\pPrefix(\str) \defeq \sum_{\str' \in \strings} \p(\str\,\str') 
\vphantom{\begin{cases} \frac{\p}{\p} \\[0.5em] \frac{\p}{\p} \end{cases}}
\label{eq:prefix-prob}
\end{equation}
\end{minipage}\hfill
\begin{minipage}[c]{0.55\linewidth}
\begin{equation}
\pPrefix(\str' \mid \str) \defeq 
\begin{cases}
\frac{\p(\str)}{\pPrefix(\str)} & \text{if } \str' = \eos \\[0.5em]
\frac{\pPrefix(\str\,\str')}{\pPrefix(\str)} & \text{otherwise}
\end{cases}
\label{eqn:conditional-prefix}
\end{equation}
\end{minipage}
\end{subequations}

where $\eos$ denotes a distinguished \defn{end-of-string} delimiter, and $\pPrefix(\cdot \mid \str)$ is a probability distribution over $\alphabet \cup \{\eos\}$. The probability of $\str$ under $\p$ then factorizes as
\begin{align}
\p(\str) &= \pPrefix(\eos \mid \str) \prod_{t=1}^{|\str|} \pPrefix(\sym_t \mid \str_{< t})
\label{eq:factored-lm}
\end{align}
corresponding to the autoregressive generative process of sampling next tokens until $\eos$ is reached.

\begin{table}[t!]
\centering
\caption{The generalized means family of $\func$-ensembles. $\ensemble$ minimizes $\sum_{k=1}^K w_k D_\alpha(\ensemble \mid \mid \p_k)$. The rows $\tau = \pm\infty$ are the limits of these minimizers and minimize $\max_k D_\infty$ instead and the $\chi^2$ rows hold up to a constant factor.}
\label{tab:f-ensembles}
\small
\setlength{\tabcolsep}{6pt}
\renewcommand{\arraystretch}{2.3}
\begin{tabular*}{\linewidth}{@{\extracolsep{\fill}} c @{\hspace{6pt}} c l l l @{}}
\toprule
& $\tau$ & Name & $\func$-Ensemble $\ensemble$ & Minimized $\alpha$-Divergence \\
\midrule

\multirow{3}{*}{\rotatebox{90}{\textcolor{gray}{\textit{consensus}}}} 
& $-\infty$ & Minimum 
& $\frac{1}{Z}\min_{k=1}^K \p_k(\charseq)$ 
& $D_\infty(\ensemble \mid\mid \p_k)=\log \max_{\charseq \in \charStrings} \frac{\ensemble(\charseq)}{\p_k(\charseq)}$ \\

& $-1$ & Harmonic 
& $\frac{1}{Z}\left(\sum_{k=1}^K \frac{w_k}{\p_k(\charseq)}\right)^{-1}$ 
& $D_{\chi^2}(\ensemble \mid\mid \p_k)=\sum_{\charseq \in \charStrings} \frac{(\ensemble(\charseq)-\p_k(\charseq))^2}{\p_k(\charseq)}$ \\

& $0$ & Product 
& $\frac{1}{Z}\prod_{k=1}^K \p_k(\charseq)^{w_k}$ 
& $D_{\mathrm{KL}}(\ensemble \mid\mid \p_k)=\sum_{\charseq \in \charStrings} \ensemble(\charseq)\log \frac{\ensemble(\charseq)}{\p_k(\charseq)}$ \\

\greydashedrule

\multirow{3}{*}{\rotatebox{90}{\textcolor{gray}{\textit{coverage}}}} 
& $1$ & Sum 
& $\frac{1}{Z}\sum_{k=1}^K w_k \p_k(\charseq)$ 
& $D_{\mathrm{KL}}(\p_k \mid\mid \ensemble)=\sum_{\charseq \in \charStrings} \p_k(\charseq)\log \frac{\p_k(\charseq)}{\ensemble(\charseq)}$ \\

& $2$ & Quadratic 
& $\frac{1}{Z}\sqrt{\sum_{k=1}^K w_k \p_k(\charseq)^2}$ 
& $D_{\chi^2}(\p_k \mid\mid \ensemble)=\sum_{\charseq \in \charStrings} \frac{(\p_k(\charseq)-\ensemble(\charseq))^2}{\ensemble(\charseq)}$ \\

& $+\infty$ & Maximum 
& $\frac{1}{Z}\max_{k=1}^K \p_k(\charseq)$ 
& $D_\infty(\p_k \mid\mid \ensemble)=\log \max_{\charseq \in \charStrings} \frac{\p_k(\charseq)}{\ensemble(\charseq)}$ \\
\bottomrule
\end{tabular*}
\end{table}

\section{Language Model \texorpdfstring{$\func$}{f}-Ensembles}
Ensembling is a technique for combining \emph{potential functions} (which may or may not be probability distributions) into a probability distribution that captures aspects of each. In particular,
given a set of potentials
$\p_{1},\ldots,\p_{\K}\colon \charStrings \rightarrow \Rnonneg$, defined over the byte alphabet $\charAlphabet$,\footnote{Note that, when $\{\p_{k}\}_{k=1}^{K}$ are defined over different token alphabets, a common issue is \textit{tokenization mismatch}, i.e., when two potentials do not share the same alphabet, token level probabilities cannot be aggregated directly. We address this by mapping each model to a shared probability space, defined on the common byte alphabet $\charAlphabet$. This problem was discussed in detail by \citet{vieira2025from}. For further details regarding the specific setting of this paper, see \cref{sec:tokenization}.} and the \emph{ensemble function} $\func\colon \RnonnegK \rightarrow \Rnonneg$, we define the $\func$-\emph{ensemble} as follows:

\begin{definition}\label{def:fensemble}
The \defn{$\func$-ensemble} of potentials $\p_1,\ldots,\p_K \in \Rnonneg^{\charStrings}$ is a distribution $\ensemble \in \simplexFun{\charStrings}$ 
\begin{equation}
\ensemble(\charseq) \defeq \frac{\func(\p_{1}(\charseq), \ldots ,\p_{K}(\charseq))}{Z}, 
\qquad
\quad Z \defeq  \sum_{\charseq' \in \charStrings} \func(\p_{1}(\charseq'),\ldots, \p_{K}(\charseq')) 
\end{equation}
for every $\func\colon \RnonnegK \rightarrow \Rnonneg$, such that $0< Z< \infty$. We further define $\targetFn(\charseq) \defeq \func(\p_{1}(\charseq), \ldots ,\p_{K}(\charseq))$ as the \textbf{unnormalized} ensemble distribution.
\end{definition}
Since $\ensemble$ is a probability distribution, the prefix probability $\ensemblePrefix$ and the conditional prefix probability $\ensemblePrefix(\charseq' \mid \charseq)$ follow from \cref{eq:prefix-prob}.
Ideally, one could sample from $\ensemble$ by iteratively sampling from the conditional next-symbol probability $\ensemblePrefix( \chars_{t} \mid \xx_{<t})$, with $\chars_{t}\!\in\!(\charAlphabet\!\cup\!\eos)$, until the end of sequence delimiter $\eos$ is hit.
However, sampling directly from $\ensemblePrefix( \chars_{t} \mid \xx_{<t})$ is often unfeasible as computing $Z$ is intractable. Therefore, in \cref{sec:approx-inference} we explore strategies to approximately sample from $\ensemblePrefix(\chars_{t} \mid \xx_{<t})$.

\subsection{A Family of \texorpdfstring{$\func$}{f}-Ensembles}
\Cref{def:fensemble} admits a wide range of aggregation functions $\func$, but which should we choose? We show that the \textit{family of generalized means} arises naturally from a variational principle: they are the unique minimizers of weighted sums of $\alpha$-divergences to the base potentials.

Specifically, consider $K$ expert potentials $\{\pExpert_k\}_{k=1}^K$, where 
$\pExpert_k\colon \charStrings \to \Rnonneg$. 
Our goal is to construct a single ensemble distribution $\ensemble$ that synthesizes the information provided by the $K$ experts, given prior expert weights $\expertPriorWeightVec \in \simplexFun{K}$.
Following prior work by \citet{dekoninck2024arithmetic}, we define the \emph{optimal} consensus distribution $\ensemble^* (\charseq)$ minimizing a weighted sum of divergences $D$, i.e.,
\begin{equation}\label{eq:variational-ensembling}
    \ensemble^* \defeq \argmin_{\ensemble\in \simplexFun{\charStrings}} \sum_{k=1}^K \expertPriorWeight_k D(\ensemble \mid \mid \pExpert_k)
\end{equation}

The choice of $D$ determines how similarity between distributions is assessed and, therefore, influences which ensemble distribution is considered optimal. In this work, we discuss $D$ within the family of $\alpha$-divergences \cite{amari1985differential}:
\begin{equation}\label{eq:alpha-divergences}
    D_\alpha (\q \mid \mid \p) = \frac{1}{\alpha(1-\alpha)} \left(1- \sum_{\charseq\in \charStrings} \q(\charseq)^\alpha \p(\charseq)^{1-\alpha}\right)
\end{equation}

\begin{restatable}[$\alpha$-Divergence Ensembles]{theorem}{alphaEnsembles}
\label{thm:alpha_ensembles}
Let $D_\alpha$ be the $\alpha$-divergence with parameter $\alpha \in \mathbb{R} \setminus \{0, 1\}$, let $\expertPriorWeight_k > 0$ for all $k$, and assume $0 < Z < \infty$. The unique consensus distribution $\ensemble^*$ that minimizes the expert-weighted loss in \Cref{eq:variational-ensembling} is the \textbf{generalized mean} of the experts, with power parameter $\tau \defeq 1 - \alpha$:
\begin{equation}\label{eq:generalized-means}
\ensemble^* (\charseq) = \frac{1}{Z} \left( \sum_{k=1}^K \expertPriorWeight_k \pExpert_k(\charseq)^{\tau} \right)^{\frac{1}{\tau}}
\end{equation}
where $Z$ is the normalizer ensuring $\sum_{\charseq\in \charStrings} \ensemble^*(\charseq) = 1$.
\end{restatable}

\begin{proof}
In \autoref{sec:proofs}.
\end{proof}

\myparagraph{Qualitative analysis of generalized mean ensembles.}
Generalized means unify many practically used model aggregation strategies under a single framework, recovering product of experts \citep{hinton1999products,hinton2002training} as $\tau \rightarrow 0$, mixture of experts \citep{jordan1994mixtures} as $\tau \rightarrow 1$, and $\min$ and $\max$ aggregations \citep{kittler1998combining} as $\tau \rightarrow \pm \infty$.
Deriving them as minimizers of weighted sums of $\alpha$-divergences further explains how these methods mediate disagreements between experts, and provides a qualitative guide for choosing $\func$ (cf. \Cref{tab:f-ensembles}).
The regime where $\tau \le 0$ exhibits \textit{consensus-seeking} behavior (where $\p_{k}(\charseq) = 0 \to \ensemble(\charseq) = 0$), which forces the ensemble to concentrate probability mass on the intersection of the expert supports. 
In contrast, the regime where $\tau > 0$ exhibits \textit{coverage-seeking} behavior (where $\p_{k}(\charseq) > 0 \to \ensemble(\charseq) > 0$), forcing the ensemble to distribute mass over the union of the expert supports. 
In both regimes, $|\tau|$ modulates the effect's strength.\looseness=-1

\section{Approximate Inference}\label{sec:approx-inference}

We now describe how to (approximately) sample from any $\func$-ensemble distribution $\ensemble$. The strategies we present are different instantiations of \emph{importance sampling} \citep{Casella-MC}. In general, we will approximately sample from $\ensemble$---which is generally intractable---by sampling from a tractable \emph{proposal} distribution and appropriately correcting the resulting distortion with an \emph{importance weight}. Adopting the importance sampling terminology, we will refer to $\ensemble$ as the \emph{target distribution}, and to $\targetFn$ as the \emph{unnormalized target}.

\subsection{Importance Sampling}
Importance sampling is a general-purpose technique for approximating intractable distributions. We now present a version adapted to the special case of estimating $\func$-ensembles:

\begin{definition}
\label{def:importance-sampling}
\defn{Importance sampling} works as follows: Given a sample budget $M>0$,\footnote{\citet{chatterjee2018importance} show that the number of samples required for importance sampling with proposal distribution $\proposal$ to obtain a \emph{good} approximation of the target distribution $\ensemble$ is exponential in the KL divergence $\KL(\ensemble \| \proposal)$.} we sample $\charseq^{(1)}, \ldots, \charseq^{(M)} \simIID \proposal$ from a \defn{proposal distribution} $\proposal \in \simplexFun{\charStrings}$.\footnote{\label{absolute-continuity}For technical reasons, we require $\proposal$ to satisfy $\targetFn(\charseq) > 0 \Longrightarrow \proposal(\charseq) > 0$ for all $\charseq \in \charStrings$, or equivalently (by contrapositive) $\proposal(\charseq) = 0 \Longrightarrow \targetFn(\charseq) = 0$ for all $\charseq \in \charStrings$. We provide a discussion about the implications in \cref{app:absolute-continuity-f}.}
We associate an \defn{importance weight} with each $\charseq^{(m)}$, i.e., $\importanceWeight^{(m)} \defeq \targetFn(\charseq^{(m)})/\proposal(\charseq^{(m)})$. The \defn{estimated unnormalized target} $\widehat{\targetFn}(\charseq)$ and \defn{estimated target distribution} $\widehat{\ensemble}(\charseq)$ are:

\begin{subequations}
\noindent
\begin{minipage}[c]{0.58\linewidth}
\begin{equation}
\widehat{\targetFn}(\charseq) \defeq \frac{1}{M} \sum_{m=1}^M w^{(m)} \indicator{\charseq = \charseq^{(m)}}
\vphantom{\frac{1}{M} \sum_{m=1}^M w^{(m)}}
\label{eq:estimate-target}
\end{equation}
\end{minipage}%
\begin{minipage}[c]{0.38\linewidth}
\begin{equation}
\widehat{\ensemble}(\charseq) \defeq \frac{\widehat{\targetFn}(\charseq)}{\widehat{Z}}
\vphantom{\frac{1}{M} \sum_{m=1}^M w^{(m)}}
\label{def:estimate-ensemble}
\end{equation}
\end{minipage}
\end{subequations}

where $\widehat{Z} \defeq \frac{1}{M}\sum_{m=1}^M w^{(m)}$ is the \defn{estimated normalization constant}.
\end{definition}

Crucially, importance sampling provides an \emph{unbiased} estimate $\widehat{Z}$ of the normalizing constant $Z$, and an estimate $\widehat{\ensemble}$ of the ensemble distribution $\ensemble$ that is \emph{consistent} (but biased for finite $M$). See \cref{app:importance-sampling} for more details and proofs of correctness. Importance sampling of complete strings is not well suited to an autoregressive language model, which is inherently sequential. For this reason, we introduce sequential importance sampling.

\subsection{Sequential Importance Sampling}\label{sec:sis}

Sequential importance sampling \citep[SIS;][]{del-moral} is a variant of importance sampling designed for distributions over sequences, where samples are built incrementally by extending partial sequences according to a proposal kernel. In language modeling, this kernel is typically the next-token distribution (cf. \cref{eq:factored-lm}), which extends a partial string by one symbol at each step. At each time step, the importance weights are updated, yielding in the end an importance weight for each \emph{particle}\footnote{Samples are often called \emph{particles} in SIS.}
as in \cref{def:importance-sampling}.
To compute the weight update, we need access to the next-token prefix target $\targetFnPrefix(\chars' \mid \charseq)$, defined analogously to \cref{eqn:conditional-prefix} with $\targetFn$ in place of $\p$. Unfortunately, $\targetFnPrefix$ is itself intractable,\footnote{It would require computing an infinite sum $\sum_{\charseq' \in \charStrings}\func(\p_1(\charseq\, \charseq'),\ldots,\p_{K}(\charseq\, \charseq'))$.} so we instead introduce a \defn{shaping function} $\shapingPrefix\colon \charStrings \to \Rnonneg$ as a tractable surrogate that guides the sampler towards promising regions of $\charStrings$. Throughout, we use $\shapingPrefix(\charseq) \defeq \func(\pPrefix_{1}(\charseq), \ldots, \pPrefix_{K}(\charseq))$, i.e., $\func$ applied to the experts' prefix probabilities.
Using the conditional shaping function $\shapingPrefix(\chars' \mid \charseq)$, $\targetFn$ factorizes as follows, $\forall \charseq \in \charStrings$:\looseness=-1
\vspace{-0.5\baselineskip}
\begin{subequations}
\noindent%
\begin{minipage}[c]{0.5\linewidth}
\begin{equation}
\targetFn(\charseq) = \shapingPrefix(\emptystr) \shapingPrefix(\eos \mid \charseq) \prod_{t=1}^{|\charseq|} \shapingPrefix(\chars_t \mid \charseq_{< t})
\label{eq:factored-potential}
\end{equation}
\end{minipage}%
\begin{minipage}[c]{0.5\linewidth}
\begin{equation}
\shapingPrefix(\chars' \mid \charseq) \defeq \begin{cases}
\frac{\targetFn(\charseq)}{\shapingPrefix(\charseq)} & \text{if } \chars' = \eos\\[0.5em]
\frac{\shapingPrefix(\charseq\,\chars')}{\shapingPrefix(\charseq)} & \text{otherwise}
\end{cases}
\label{eq:conditional-shaping}
\end{equation}
\end{minipage}
\end{subequations}

\begin{proposition}
\label{prop:shaped-weights}
The shaping weights are equivalent to those used in importance sampling:
\begin{align}
\frac{\targetFn(\charseq)}{\proposal(\charseq)} 
&= \shapingPrefix(\emptystr) \frac{\shapingPrefix(\eos \mid \charseq)}{\proposalPrefix(\eos \mid \charseq)} \prod_{t=1}^{|\charseq|} \frac{\shapingPrefix(\chars_t \mid \charseq_{< t})}{\proposalPrefix( \chars_t\mid\charseq_{< t}) }
\end{align}
\end{proposition}
\begin{proof}
\Cref{prop:shaped-weights} follows directly from \cref{eq:factored-potential} and \cref{eq:factored-lm} applied to $\proposal(\charseq)$, followed by algebra.
\end{proof}
\begin{restatable}{proposition}{optimalProposal}\label{prop:optimal-proposal}
The locally optimal proposal distribution (i.e., the one that minimizes the variance of the importance weight update at each step) is $\proposalPrefix^*(\chars' \mid \charseq) \propto \shapingPrefix(\chars' \mid \charseq)$. 
\end{restatable}
\noindent See \cref{prop:optimal-proposal:proof} for proof. Note that the two may still differ when the shaping function is not normalized, since $\shapingPrefix(\cdot \mid \charseq)$ is not necessarily a probability distribution.

\begin{algorithm}[t!]
\caption{Sequential Monte Carlo\protect\footnotemark}
\label{alg:smc}
\begin{minipage}[t]{0.53\linewidth}
\begin{algorithmic}[1]
\footnotesize
\Procedure{SMC}{$\proposal, \shapingPrefix, M, \essThreshold$}
\For{$m = 1 \ldots M$}    \Comment{Initialize $M$ particles at $\varepsilon$}
  \State $(\charseq^{(m)}, w^{(m)}, \activeParticle^{(m)}) \gets (\varepsilon, \shapingPrefix(\varepsilon), \texttt{true})$
\EndFor
\While{$\exists m \in 1 \ldots M\colon \activeParticle^{(m)}$} \Comment{Until all particles hit $\eos$}
\For{$m=1 \dots M \text{ such that } \activeParticle^{(m)}$}
\State $y' \sim \proposal(\cdot \mid \charseq^{(m)})$ \Comment{Extend from proposal}
\State $w^{(m)} \gets w^{(m)} \frac{\shapingPrefix(y' \mid \charseq^{(m)})}{\proposal(y' \mid \charseq^{(m)})}$ \Comment{IS update, \cref{eq:conditional-shaping}}
\If{$y' = \eos$}
  \State $\activeParticle^{(m)} \gets \texttt{false}$ \Comment{Freeze finished particle}
\Else
  \State $\charseq^{(m)} \gets \charseq^{(m)} \circ y'$
\EndIf
\EndFor
\State $(\charseq^{(\cdot)}, w^{(\cdot)}, \activeParticle^{(\cdot)})
\gets \textsc{Resample}(\charseq^{(\cdot)}, w^{(\cdot)}, \activeParticle^{(\cdot)},
\essThreshold)$
\EndWhile
\State $\Zest \gets \frac{1}{M} \sum_{m=1}^M w^{(m)}$ \Comment{Unbiased estimate of $Z$}
\State $\widehat{\targetFn}(\charseq) \gets \frac{1}{M} \sum_{m=1}^M w^{(m)} \mathds{1}\{ \charseq \!=\! \charseq^{(m)}\}$
\State $\widehat{\ensemble}(\charseq) \gets \widehat{\targetFn}(\charseq) / \Zest$ \Comment{Consistent estimate of $\ensemble$}
\State \Return $(\Zest, \widehat{\targetFn}, \widehat{\ensemble})$
\EndProcedure
\end{algorithmic}
\end{minipage}%
\hfill
\begin{minipage}[t]{0.47\linewidth}
\vspace{5ex}
\begin{algorithmic}[1]
\footnotesize
\Procedure{Resample}{$\charseq^{(\cdot)}, w^{(\cdot)}, \activeParticle^{(\cdot)}, \essThreshold$}
\State $\totalWeight \gets \sum_{m=1}^M w^{(m)}$
\State $\ess \gets \totalWeight^2 / \sum_{m=1}^M \big(w^{(m)}\big)^2$ \Comment{Eff.\ sample size}\label{line:ess-def}
\If{$\ess < \essThreshold \!\cdot\! M$} \Comment{Trigger resampling}
\State $\overline{\charseq}^{(\cdot)} \gets \charseq^{(\cdot)};\ \tmpWeight^{(\cdot)} \gets w^{(\cdot)};\ \overline{\activeParticle}^{(\cdot)} \gets \activeParticle^{(\cdot)}$
\For{$m = 1 \ldots M$} \Comment{Resample, reset weights}
\State $R \sim \mathrm{Categorical}(\totalWeight^{-1}\langle \tmpWeight^{(1)}, \ldots, \tmpWeight^{(M)}\rangle)$
\State $(\charseq^{(m)}, w^{(m)}, \activeParticle^{(m)}) \gets (\overline{\charseq}^{(R)}, \frac{\totalWeight}{M}, \overline{\activeParticle}^{(R)})$ 
\EndFor
\EndIf
\State \Return $(\charseq^{(\cdot)}, w^{(\cdot)}, \activeParticle^{(\cdot)})$
\EndProcedure
\end{algorithmic}
\end{minipage}
\end{algorithm}
\footnotetext{Note that the final weight step is more detailed than the pseudocode suggests, as $\eos$ is handled as in \cref{eq:conditional-shaping}.}

\subsection{Sequential Monte Carlo}
Sequential Monte Carlo (SMC) is a generalization of SIS that incorporates a resampling step that reallocates computational effort from less to more promising partial sequences during particle sampling.
More concretely, using a multinomial resampling strategy, at each step,\footnote{In practice, resampling is only done based on the  \defn{effective sample size} 
$\ess \defeq \left(\sum_{m=1}^M w^{(m)}\right)^2 /\sum_{m=1}^M \left(w^{(m)}\right)^2$
which measures the variance among importance weights. 
We resample when $\ess$ falls below a threshold $\essThreshold\cdot M$ for $\essThreshold \in (0, 1)$.} SMC samples indices $a^{(1)}, \ldots, a^{(M)} \simIID \mathrm{Categorical}(\totalWeight^{-1}\langle w^{(1)}, \ldots, w^{(M)}\rangle)$, where $\totalWeight \defeq \sum_{m=1}^M w^{(m)}$. All particles are then reassigned as $(\charseq^{(m)}, w^{(m)}) = (\charseq^{(a^{(m)})}, W/M)$. 
The full procedure is shown in \Cref{alg:smc}. With resampling, $\Zest$ remains unbiased and $\widehat{\ensemble}$ consistent \citep{delmoral2004feynman, chopin2020introduction}.

\myparagraph{Token- and byte-level ensembles.} \Cref{alg:smc} only requires that all experts assign next-symbol probabilities over a shared alphabet. When the experts share a tokenizer, we therefore run it directly over tokens. When their tokenizers differ, we run it over bytes. To this end, we convert each expert into a language model over byte strings by marginalizing over all token sequences that decode to the same string (cf.\ \cref{sec:tokenization}), and compute its next-byte probabilities with the beam-based approximation of \citet{vieira2025from}.

\section{Experiments}\label{sec:experiments}
Different LMs and prompts often exhibit complementary strengths on the same task \citep{polo2024efficient}. As discussed in our related work section (cf. \Cref{sec:related-work}), prior work on LM ensembling during inference has largely focused on local probability averaging (i.e., sum ensembles) \citep{yao2025determinethenensemble}, whereas our $\func$-ensemble framework spans a broader family of aggregation strategies. We evaluate whether these alternative $\func$-ensembles can exploit this to improve upon individual models across tasks. More specifically, our experiments aim to answer the following questions:
\begin{enumerate}[label=RQ\arabic*., leftmargin=*, labelwidth=2em]
    \item To what extent can LMs be \textit{synergistic}, i.e., work better jointly, rather than on their own?
    \item In what ways does the mediation strategy between models, encoded by $\func$, affect performance under the resulting ensemble?
    \item How does the quality of approximation of $\ensemble$ affect the performance of the $\func$-ensemble?
\end{enumerate}

\subsection{Setup}

\myparagraph{Models and Prompts.}
We evaluate ensembles using instruction-tuned LMs from three model families: Llama \citep[Llama-3.1-8B-Instruct;][]{grattafiori2024llama3herdmodels}, Qwen \citep[Qwen2.5-7B-Instruct;][]{qwen2025qwen25technicalreport}, and Phi \citep[Phi-4 (14B);][]{abdin2024phi4technicalreport}. 
This yields two experimental settings: \textit{within-model} ensembles that combine the same model under different prompts and \textit{cross-model} ensembles that combine two models under the same prompt. Within-model ensembles share a tokenizer and are run at the token level, whereas cross-model ensembles are run at the byte level (cf.\ \cref{sec:approx-inference}).

\myparagraph{Datasets.}
We study the effectiveness of $\func$-ensembles on five text generation tasks standard to LM ensembling \citep{yao2025determinethenensemble, yu2024breaking} and text generation with SMC \citep{loula2025syntactic}. 
\begin{itemize}
    \item \textbf{Text-to-SQL.} \textit{Task:} Generate SQL queries from a natural language question and relevant database schema \citep{yu-etal-2018-spider}. \textit{Metric: } Execution accuracy, i.e., checking whether the generated query produces the same output when executed against the provided test database.
    \item \textbf{GSM8K.} \textit{Task:} Grade school math word problems requiring multi-step arithmetic reasoning \citep{cobbe2021gsm8k}. \textit{Metric:} The final reasoning trace is parsed for various answer formats. We then evaluate binary correctness of the final numerical answer. 
    \item \textbf{Big-Bench Hard.} \textit{Task:} Diverse reasoning tasks \citep{suzgun-etal-2023-challenging}. We focus on the word sorting subtask: given a list of words, sort them in alphabetical order. \textit{Metric:} Binary correctness of the generated sorted list. 
    \item \textbf{MMLU.} \textit{Task:} Multiple-choice questions spanning 57 academic subjects \citep{hendrycks2021mmlu}. \textit{Metric:} Binary correctness of the selected answer choice.
    \item \textbf{JSON.} \textit{Task:} Generate documents that conform to a specific JSON schema using the GlaiveAI function-calling dataset \citep{glaiveai-function} from JSONSchemaBench \citep{geng2025json}. \textit{Metric:} Binary correctness, whether the generated string conforms to a specific JSON schema.
\end{itemize}

\begin{figure}[t!]
    \centering
    \includegraphics[width=0.9\linewidth]{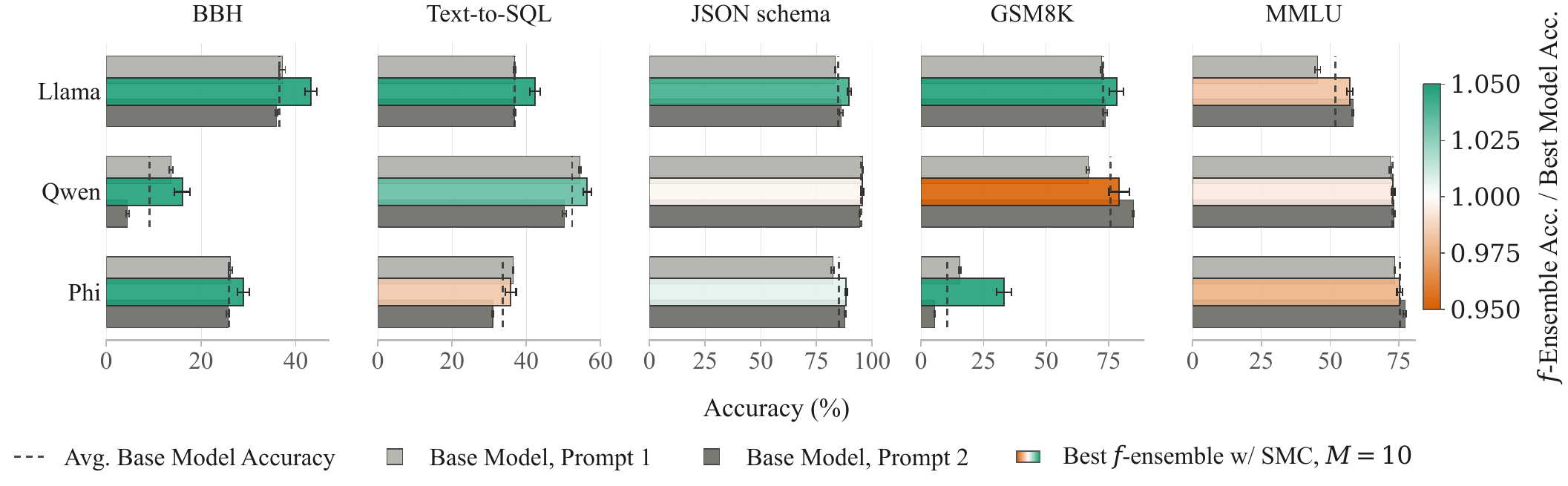}
    \caption{Two-prompt within-model ensemble vs. base model performance across tasks and models with 95\% CIs across 5 seeds. }
    \label{fig:within-model}
\end{figure}

\myparagraph{Baselines.}
We are primarily interested in whether an ensemble improves over its best-performing base model, which is therefore our main point of comparison. A comparison to naive local ensembling can be found in \cref{tab:main-results}. Although vocabulary alignment is not a central contribution of this work, we further include token alignment heuristics that employ token-level probability averaging baselines, namely UniTE \citep{yao2025determinethenensemble} and GaC \citep{yu2024breaking}, to situate our results within the broader ensembling literature. We note that we evaluate these methods in settings where they are reported to be weak, namely when the base models differ substantially in accuracy \citep{yao2025determinethenensemble}.

\myparagraph{Evaluation Metric.} Because accuracy is a random variable and a $\func$-ensemble defines a distribution over outputs, we evaluate its performance in terms of the probability mass it assigns to correct answers, i.e., accuracy in expectation under $\ensemble$, rather than the accuracy of a single decoded output. We approximate it using the weighted particles from SMC:
\begin{equation}
\E_{\charseq \sim \ensemble}\left[\indicator{\charseq \in \mathcal{Y}}\right] \approx \sum_{m=1}^{M} \bar{\importanceWeight}^{(m)} {\cdot} \indicator{\charseq^{(m)} \in \mathcal{Y}}
\end{equation}
where $\bar{\importanceWeight}^{(m)} = \importanceWeight^{(m)} / \sum_{m'=1}^M \importanceWeight^{(m')}$ are the normalized importance weights and $\mathcal{Y} \subseteq \charStrings$ is the set of correct outputs. This is standard practice in the literature on SMC for language models \citep{loula2025syntactic, lipkin2025fast} and estimates the same quantity as the accuracy of a single output sampled from the weighted particles, with lower variance.

\myparagraph{$\func$-Ensembles.} We evaluate four generalized means spanning the consensus-coverage spectrum: $\min$ of experts ($\tau \rightarrow -\infty$), product of experts ($\tau = 0$), mixture of experts ($\tau = 1$), and $\max$ of experts ($\tau\rightarrow \infty$). These represent commonly used aggregation strategies \citep{khalifa2023exploring} and the meaningful limits of the generalized mean family (see \Cref{tab:f-ensembles}), enabling systematic analysis of how the choice of aggregation affects performance.
The functions we use are \textit{annihilative}\footnote{In fact, all generalized means are, as derived in \Cref{prop:annihilative-f}.}, ensuring the used shaping function (cf. \Cref{sec:sis}) satisfies absolute continuity when used as the proposal.

Across experiments, we report means and 95\% CIs across 5 different seeds with $M=10$ particles (we provide particle ablations in \Cref{sec:app-particle-study}), a resampling threshold of 0.9 (ablated in \cref{sec:app-ess}), and equal expert weight, unless specified otherwise.
To understand the interactions between base models and the resulting ensemble distribution, we run our experiments for $K=2$ experts (cf. \cref{sec:app-k-experts} for $K>2$). 
We report more details in \cref{sec:exp-details}.
We provide a detailed runtime analysis comparing base model inference to local approximations and full SMC at different particle values in \Cref{tab:runtimes_json}.

\subsection{Results}\label{sec:results}

\begin{figure}[t!]
    \centering
    \begin{subfigure}[b]{0.5\linewidth}
        \centering
        \includegraphics[width=\linewidth]{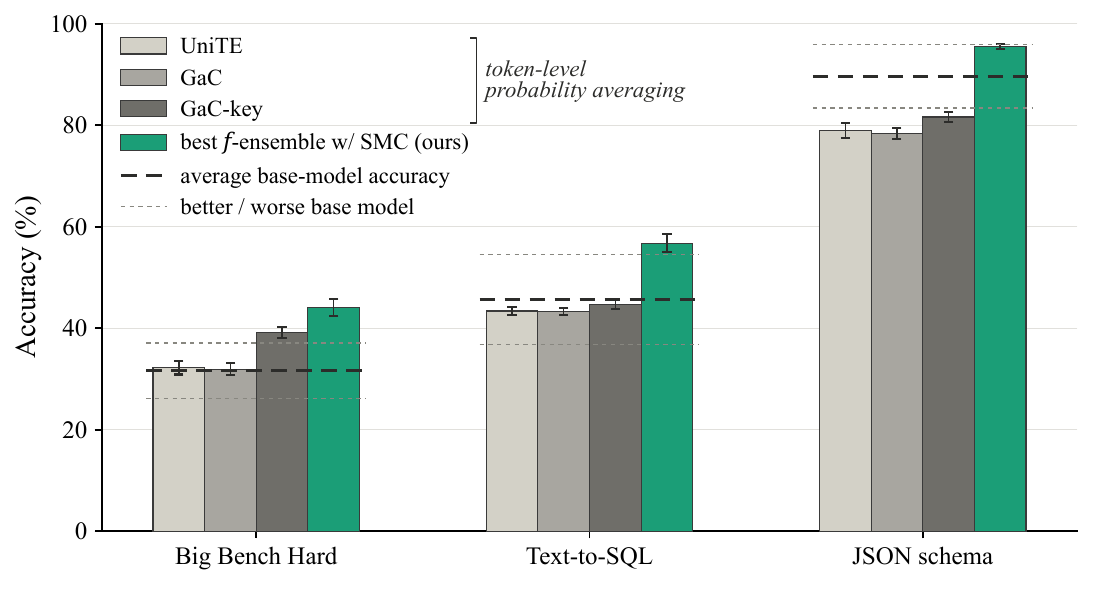}
        \caption{Cross-model: Baselines vs. $\func$-ensembles}
        \label{fig:cross-model-baselines}
    \end{subfigure}
    \hfill
    \begin{subfigure}[b]{0.4\linewidth}
        \centering
        \includegraphics[width=\linewidth]{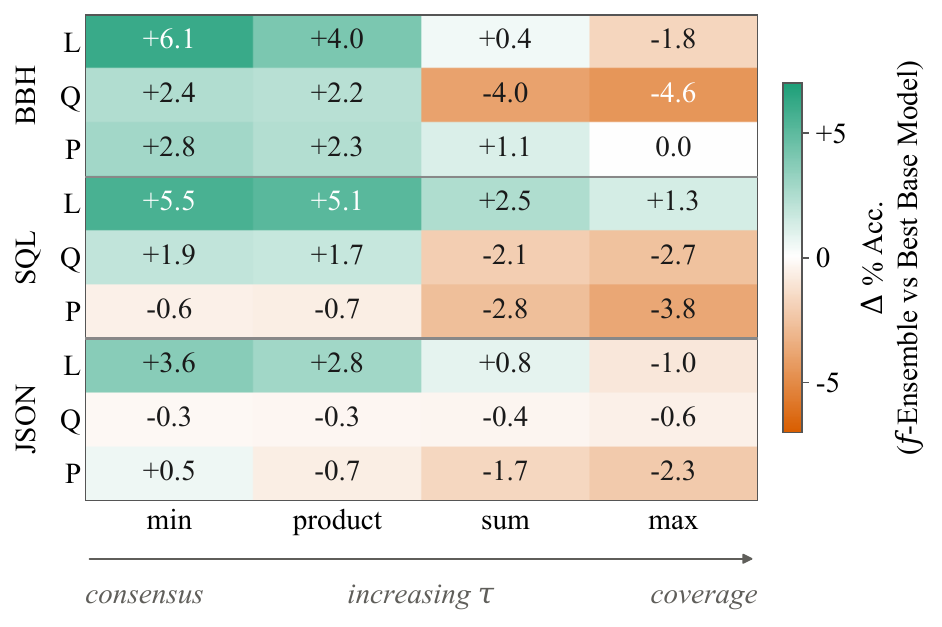}
        \caption{Within-model: Comparison between $\func$s}
        \label{fig:f-ensemble-diffs-heatmap}
    \end{subfigure}
    \caption{(\subref{fig:cross-model-baselines}) Cross-model best-pair comparison against popular ensembling baselines (UniTE @ $k=50$, GaC, GaC Key @ $t=0.9$) with 95\% CIs. 
    The dashed line marks the average of the two base-model accuracies, the theoretical ceiling for any equal-weight averaging ensemble.
    \textbf{(\subref{fig:f-ensemble-diffs-heatmap})} Within-model $\func$-ensembling for different $\func$, using $M=10$ particles: change in accuracy (\%) relative to the best single-prompt baseline. Q=Qwen, P=Phi, L=Llama. Additional results for more datasets are shown in \Cref{sec:appendix-additional-results}, \cref{tab:func-comparison} and \cref{tab:gsm8k-mmlu-comparison}.}
    \label{fig:f-ensemble-comparisons}
\end{figure}

\myparagraph{[RQ1] Models can act synergistically with themselves and each other.}
As shown in \cref{fig:within-model}, $\func$-ensembling yields consistent and significant improvements over the best-performing base model across tasks and model families, with a notable exception for the multiple-choice dataset MMLU.
We hypothesize that ensembling is most beneficial when the experts are complementary, i.e., when both perform moderately on the same examples, the ensemble can concentrate probability on their intersection; when one already succeeds or both fail, there is little to gain.\footnote{We provide empirical support in \cref{sec:app-overlap}.}

Similar gains are observed for cross-model ensembles (cf.\ \cref{tab:main-results}), suggesting that architectural and data diversity provide useful variation for ensembling beyond prompting alone.

\myparagraph{[RQ2] Consensus-seeking ensembles improve expected accuracy over probability averaging.}
As shown in \cref{fig:cross-model-baselines}, consensus-seeking $\func$-ensembles consistently outperform token-level probability averaging baselines, including UniTE and GaC, which cluster below the average base model accuracy ceiling. 
We note that this can be expected theoretically, and is independent of any token alignment method. Namely, for any weighted mixture $\ensemble(\charseq) = \sum_{k=1}^K w_k \p_k(\charseq)$ with $w \in \simplexFun{K}$, the ensemble's expected accuracy is bounded by the better-performing base model:
\begin{equation}
\E_{\charseq \sim \ensemble} [\indicator{\charseq \in \mathcal{Y}}] 
= 
\smashoperator{\sum_{\charseq \in \charStrings}} \ensemble(\charseq) {\cdot} \indicator{\charseq \in \mathcal{Y}} 
= \sum_{k=1}^K w_k {\cdot} \smashoperator{\E_{\charseq \sim \p_k}}\ [\indicator{\charseq \in \mathcal{Y}}] \leq \max_{k=1}^K \ \smashoperator{\E_{\charseq \sim \p_k}}\ [\indicator{\charseq \in \mathcal{Y}}]
\end{equation}
For equal weighting, $\E_{\charseq \sim \ensemble}[\indicator{\charseq \in \mathcal{Y}}]$ equals the arithmetic mean of base model accuracies. 
First, this provides a sanity check of our implementation: using our SMC algorithm, we observe that the expected accuracy of the equally-weighted mixture approaches the average accuracy of the base models as we increase the number of particles (cf. \cref{fig:particle-study-eacc}).
Second, this observation yields a theoretical explanation for prior empirical work by \citet{yao2025determinethenensemble}, which shows that if base LMs exhibit a large performance gap, a variety of LM ensembling methods that rely on local probability averaging drop below the accuracy of the better-performing model.
Turning to the comparison among $\func$-ensembles (cf.\ \cref{fig:f-ensemble-diffs-heatmap}), the primary distinction is between consensus-seeking and coverage-seeking operators, with $\min$ ensembles showing the best performance overall, showing to be more robust compared to coverage-seeking aggregation strategies.

\myparagraph{[RQ3] Approximation quality affects $\func$-ensemble performance differently.}
While \cref{fig:cross-model-baselines} already shows that SMC $\func$-ensembles improve over baseline approaches, it remains unclear in what scenarios improved \textit{approximation quality} (i.e., the closeness of our sampling distribution to the target $\ensemble$) correlates with expected accuracy.
We measure approximation quality via $\E[\log \widehat{Z}]$, which following \citet{andrieu2009thepa, lew2023sequential, zhao2024twisted} provides a lower bound on $\log Z - \mathrm{KL}(\q_{\mathrm{alg}} \| \ensemble)$, where $\q_{\mathrm{alg}}$ is the marginal distribution induced by the SMC algorithm. Intuitively, higher values indicate samples closer to the global ensemble $\ensemble$. We show the approximation quality improves as the number of particles increases in \cref{sec:app-particle-study}.

\begin{figure*}[t!]
    \centering
    \includegraphics[width=\linewidth]{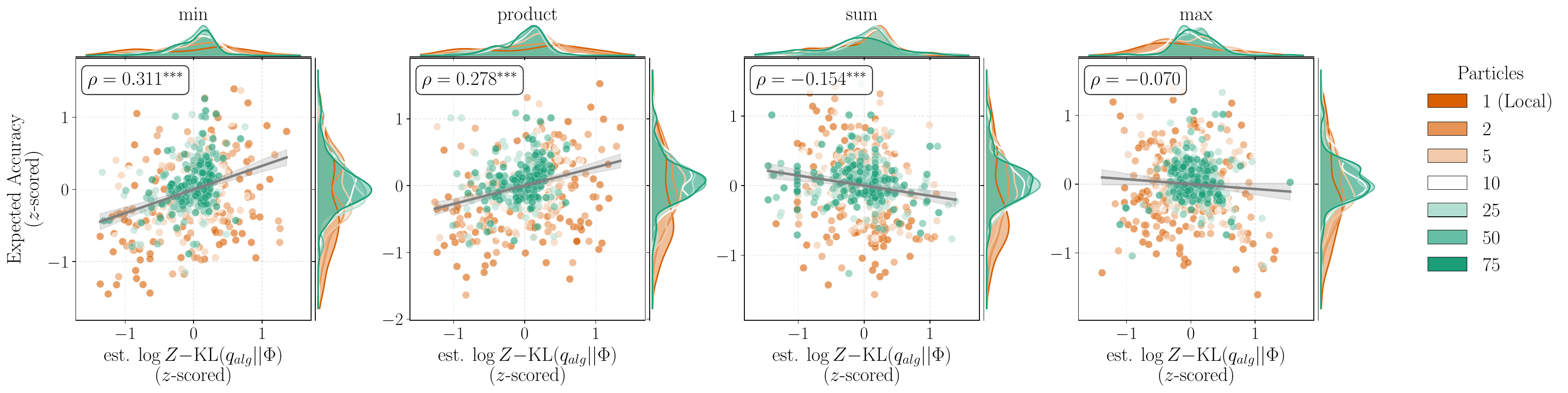}
    \caption{Representative relationship between expected accuracy and the estimated log marginal likelihood, $\log \hat Z$, on BBH for Llama. Each data point represents (sample, particle) configurations aggregated across seeds. Both metrics are $z$-scored per example to normalize for variations in problem difficulty and scale. Points are colored by the number of particles used for estimation. Marginal plots display the kernel density estimates for the respective axes. The grey line and shaded band indicate a linear regression fit with a 95\% confidence interval. Pearson correlation coefficients $\rho$ are annotated in the top-left ($^{***}p<0.001$, $^{**}p<0.01$, $^{*}p<0.05$).
    }
    \label{fig:approximation-quality-bbh}
\end{figure*}

As shown in \cref{fig:approximation-quality-bbh}, consensus-seeking operators (min and product) mostly exhibit significant positive correlations between approximation quality and expected accuracy (see also \cref{sec:appendix-additional-results} for full results). This suggests that for these $\func$-ensembles, better posterior approximations (whether through additional particles or improved proposals) translate to measurably higher task performance.
In contrast, coverage-seeking operators (sum and max) show weak correlations that are mostly negative. For sum ensembles, this is expected, since the expected accuracy of the global mixture equals the average of the base models' accuracies and therefore cannot improve with better inference.

\section{Conclusion}\label{sec:conclusion}
We develop $\func$-ensembling; a unified framework for composing language models into ensemble distributions over strings.
Motivated by significant differences in locally and globally normalized string scoring in prompt intersection, we further propose a flexible byte-level SMC ensembling algorithm that prioritizes globally promising samples during inference. 
In a series of experiments, we show that language models can act synergistically with one another. Unlike probability averaging, which is bounded by its base model performances, consensus-seeking $\func$s yield more robust accuracy improvements and, for these operators, better posterior approximations translate to better task performance.\looseness=-1

\begin{ack}
The authors would like to thank the members of the extended CHI-FRO team, and specifically Alexander K. Lew, Joshua B. Tenenbaum, and Benjamin LeBrun, for their help in conceiving this research and for valuable discussions, feedback, and programming support that improved this work.
This work was supported by CHI-FRO.
\end{ack}

\bibliography{custom}
\bibliographystyle{acl_natbib}

\onecolumn
\appendix
\raggedbottom

\section{Related Work}\label{sec:related-work}

\myparagraph{Ensembling Language Models. }A recent survey by \citet{chen2025ensemble} taxonomizes recent work around language model ensembling. Namely, they categorize these works into three major groups based on whether ensembling occurs \textit{before}, \textit{during}, or \textit{after} decoding. Methods applied \textit{before} decoding typically involve routing queries to specific models based on predicted performance or cost, such as the classification-based routing by \citet{shnitzer2024large} or reward-guided methods like Zooter \citep{lu2024routing}. Conversely, \textit{after} decoding approaches focus on selecting or fusing complete responses generated by multiple models; prominent examples include pairwise ranking in LLM-Blender \citep{jiang2023blender}, best-of-$N$ sampling \cite{lightman2024lets}, or cascade systems \citep{jitkrittum2023when, gupta2024language}. When inputs are ambiguous, such selection can also be left to the user, who chooses among sampled interpretations \citep{saparina2024ambrosia, chan2026pleasqlarify}. Finally, ensembling \textit{during} decoding involves aggregating predictions at the token or span level to guide generation, which is the subject of this study.

\myparagraph{Ensembling Language Models During Sampling.}
Existing methods for combining language models at the token level face the vocabulary alignment problem, as incompatible tokenizers yield distributions over mismatching vocabularies. To address this, recent works have proposed mapping distributions into a shared space to enable probability aggregation. For instance, \citet{yu2024breaking} introduce a union vocabulary approach, projecting distributions onto a superset of tokens from all participating models, then computing the average token probability across models. \citet{yao2025determinethenensemble} refine this by constructing a union over only the top-$k$ tokens, which reduces computational overhead. Alternatively, approaches like DeePEn \citep{huang2024ensemble} and EVA \citep{xu2024bridging} bypass explicit vocabulary unification by projecting outputs into a shared relative embedding space or using pivot mechanisms to fuse distributions. Other techniques, such as PackLLM \citep{mavromatis2024pack}, focus on dynamic weighting, adjusting the influence of each model based on step-wise perplexity.
Most relatedly, \citet{phan2025exact} combine byte-level language models of different model families via probability averaging.
Many of these papers focus on developing heuristics for dealing with the vocabulary alignment problem, while applying averages \cite{yu2024breaking, xu2024bridging, yao2025determinethenensemble} or some type of weighted averages \cite{huang2024ensemble, mavromatis2024pack, jin2024collaborativedecodingcriticaltokens, li2024purifyinglargelanguagemodels} to the token probabilities.

In our work, we use prior methods for mapping language models to a shared byte vocabulary \citep{phan2025exact, vieira2025from}, which avoids the alignment problem. This allows us to focus on how else language model distributions can be combined. We explore alternatives to probability averaging and propose an SMC-based ensembling algorithm that enables, in the limit, consistent sampling from the combined distribution.

\myparagraph{Approximate Posterior Inference for Large Language Models During Sampling.}
Constraints on LM outputs, such as grammars or schemas, are most commonly enforced by masking invalid next tokens at each step \citep{poesia2022synchromesh, chan-etal-2024-adapting, geng2025json}. Like token-level ensembling, this masking normalizes locally and does not sample from the global conditional distribution \citep{lew2023sequential, loula2025syntactic}.
Prior work formulates constrained generation from language models as posterior inference and employs approximate inference algorithms to sample from intractable target distributions. Several approaches have been explored, including rejection sampling \citep{poesia2022synchromesh, lipkin2025fast}, MCMC methods \citep{miao2019cgmh, zhang2020language, qin2022cold}, and sequential Monte Carlo \citep{lew2023sequential, zhao2024twisted, loula2025syntactic}. Of these, SMC fits language models well, as it exploits their autoregressive factorization and does not need to re-evaluate entire sequences after each modification.
\citet{lew2023sequential} introduced SMC steering of LMs via probabilistic programs, enabling globally consistent sampling from conditional targets while only computing local constraints. \citet{zhao2024twisted} developed learned twist functions via contrastive training to guide SMC toward high-value completions, along with bidirectional bounds on the log-partition function for evaluating inference quality. Most recently, \citet{loula2025syntactic} applied SMC to sample from the global product of a language model and various syntactic and semantic potentials, demonstrating improved performance from better posterior approximations.

These prior works focus on combining a \emph{single} language model with external constraints (e.g., grammars, semantic specifications). In contrast, we study the composition of \emph{multiple} language models via general aggregation functions $\func$, where the target distribution arises from the models themselves rather than auxiliary constraints. In this setting, the choice of $\func$ induces qualitatively different ensemble behaviors, and we study how $\func$ and approximation quality affect performance.

\section{Broader Impact}\label{sec:broader-impact}
This paper presents work aimed at advancing the field of
machine learning by combining language models. We do not
foresee direct paths to harmful applications specific to this work beyond those that apply to language models generally.

\section{Declaration of LLM Usage}\label{app:llm-usage}
Large language models were used for assistance with editing the manuscript, debugging
code, and checking LaTeX formatting. All scientific content, theoretical
results, experimental design, and conclusions are the work of the authors.

\section{Limitations}\label{sec:limitations}

\myparagraph{Computational cost.}
Our approach relies on sequential Monte Carlo, which, as noted by \citet{loula2025syntactic}, introduces inference-time overhead relative to standard decoding due to maintaining multiple particles and performing resampling. This cost is amplified at the byte level, since sequences require more SMC steps and computing byte-level next-symbol probabilities from tokenized LMs requires marginalizing over admissible tokenizations. However, we note that SMC can operate in parallel across particles.

\myparagraph{Scope of ensembling.}
As discussed in \autoref{sec:related-work}, our work focuses on ensembling language models \textit{during} decoding, i.e., combining next-token distributions at each generation step. This is distinct from other inference-time combination strategies. Model routing \citep{shnitzer2024large} selects a single model per query, avoiding the cost of running multiple models but forgoing the benefits of combining their predictions at the token level. Post-hoc methods like best-of-$N$ \citep{lightman2024lets} or LLM-Blender \citep{jiang2023blender} generate complete outputs from each model and then select or fuse among them, enabling comparison of full sequences but requiring multiple complete generations. These strategies are complementary to ours, and during-decoding ensembles could, for instance, be combined with post-hoc reranking.

\myparagraph{Model selection.}
We evaluate instruction-tuned models (Llama, Qwen, Phi) in the 7--14B parameter range. We do not evaluate reasoning-specialized models (e.g., models trained with chain-of-thought or process supervision) or larger frontier models, and our findings may not transfer to them.

\myparagraph{Task selection.}
We evaluate on structured text generation tasks (JSON schema conformance, word sorting, text-to-SQL), as well as mathematical reasoning (GSM8K) and multiple-choice question answering (MMLU), all with objective, binary correctness criteria. We find benefits from consensus-seeking ensembles and better posterior approximations on all tasks except MMLU. We do not evaluate open-ended generation tasks (e.g., creative writing, dialogue), where correctness is harder to assess.

\section{Tokenized Language Models}\label{sec:tokenization}

Modern language models are typically designed as probability distributions over strings of \emph{tokens} rather than over strings of bytes.
Although one byte string admits only one canonical tokenization, multiple sequences of tokens may correspond to the same string of bytes. This has two fundamental consequences. First, the LM may spread probability mass among multiple valid tokenizations of a byte string \cite{phan2025exact, vieira2025from,vieira2025language}.
Second, when comparing the probability that two different models assign to a given string, as we do in the \emph{ensemble} framework, one cannot directly compare token-level probabilities because the models may tokenize the string differently.\looseness=-1

To combine such models, we must map them to a common space.
As all valid tokenizations of a string ultimately decode to the string itself, one possible way is to consider the probability distribution over byte strings.\footnote{In practice, we first use a layer for mapping character strings to byte strings $\mathcal{B}^*$, where $\mathcal{B}=\{0, {\ldots}, 255\}$.}\textsuperscript{,}\footnote{Note that an alternative approach is to map a language model's distribution to another language model's tokenization space \citep{phan2025crosstokenizerlikelihoodscoringalgorithms}.}

This section formalizes this connection through the notion of a \emph{tokenized language model}.
In order to introduce this notion, we define a \emph{tokenization model} \citep{gastaldi2025the} as a tuple $\langle\charAlphabet, \tokAlphabet, \tokenizer,\decodingFun \rangle$, where $\tokAlphabet$ and $\charAlphabet$ are respectively the token and byte alphabet, $\tokenizer\colon \charAlphabet^* \rightarrow \tokAlphabet^*$ is the tokenizer, and $\decodingFun\colon \tokAlphabet^* \rightarrow \charAlphabet^*$ is a decoding function that must satisfy $\decodingFun(\tokenizer(\charseq))=\charseq$.
\begin{definition}
Let $\tokAlphabet$ be an alphabet of token symbols, $\charAlphabet$  an alphabet of byte symbols, and $\ptokens$ a language model over $\tokAlphabet^*$. A \textbf{tokenized language model} $\p_{\charAlphabet}$ is a language model over $\charStrings$ that is induced by the decoding function $\decodingFun\colon \kleene{\tokAlphabet} \rightarrow \charStrings$:
\begin{equation}
\label{eq:character-level-model}
\pchars(\charseq) 
\defeq \Pr{\Y \sim \ptokens}{\decodingFun(\Y) = \charseq}
\end{equation}
\end{definition}
Note that $\pchars(\charseq)$ correctly accounts for the fact that \emph{many} token strings $\Y$ may map to a given byte string $\charseq$ through the decoding function $\decodingFun$. The prefix probability $\pPrefix_{\charAlphabet}$ follows automatically from \cref{eq:prefix-prob,eqn:conditional-prefix}.

\section{Deferred Proofs}\label{sec:proofs}

\subsection{Importance Sampling}\label{app:importance-sampling}

\begin{restatable}{proposition}{ImportanceSamplingGuarantees}
\label{prop:importance-sampling-guarantees}
Importance sampling has the following guarantees:
\begin{enumerate}[noitemsep,topsep=0pt,parsep=0pt,partopsep=0pt,leftmargin=*]
\item The estimated normalization constant is unbiased:
\begin{align}
\E_{\charseq^{(1)}, \ldots, \charseq^{(M)} \simIID \proposal} \left[ \widehat{Z} \right]
= Z
\end{align}
\item The estimated target function is unbiased:
\begin{align}
\E_{\charseq^{(1)}, \ldots, \charseq^{(M)} \simIID \proposal} \left[
\widehat{\targetFn}(\charseq)
\right] = \targetFn(\charseq)
\end{align}
\item The estimated distribution is consistent:
\begin{align}
\lim_{M \to \infty} \E_{\charseq^{(1)}, \ldots, \charseq^{(M)} \simIID \proposal} \left[ \widehat{\ensemble}(\charseq) \right]
= \ensemble(\charseq)
\end{align}
\end{enumerate}
\end{restatable}

\label{prop:importance-sampling-guarantees:proof}
\begin{proof}
\textbf{Part (1).}
The fact that $\hat{Z}$ is an unbiased estimator of $Z$ follows from the definition of $\E$ and its linearity under the assumption that $\proposal(\charseq) = 0 \Longrightarrow \targetFn(\charseq) = 0$ for all $\charseq \in \charStrings$ \citep[e.g.,][]{ross2023variance}:
\begin{subequations}
\begin{align}
    \E_{\charseq^{(1)}, \ldots, \charseq^{(M)} \simIID \proposal} \left[ \widehat{Z} \right]   &= \E_{\charseq^{(1)}, \ldots, \charseq^{(M)} \simIID \proposal} \left[\frac{1}{M} \sum_{m=1}^{M}  w^{(m)} \right] &\proofExplain{\cref{def:importance-sampling}}\\
    &= \frac{1}{M} \sum_{m=1}^{M} \E_{\charseq^{(m)} \simIID \proposal} \left[  w^{(m)} \right] &\proofExplain{Independence}\\
    &= \frac{1}{M} \sum_{m=1}^{M}  \E_{\charseq \sim \proposal} \left[ \frac{\targetFn(\charseq)}{\proposal(\charseq)} \right]  &\proofExplain{Identicality}\\
    &= \E_{\charseq \sim \proposal} \left[ \frac{\targetFn(\charseq)}{\proposal(\charseq)} \right] &\proofExplain{Algebra}\\
    &= \sum_{\charseq\in\charStrings} \proposal(\charseq) \frac{\targetFn(\charseq)}{\proposal(\charseq)} & \proofExplain{Def.\@ of \ensuremath{\E}} \\
    &=\sum_{\charseq\in\charStrings} \targetFn(\charseq) & \proofExplain{Algebra, Absolute Continuity}\\
    &= Z &
\end{align}
\end{subequations}    

\textbf{Part (2).}
Under the same assumptions as above, we similarly show that $\hat{\targetFn}$ is an unbiased estimator of $\targetFn$. Choose $\charseq \in \charStrings$ arbitrarily.
\begin{subequations}
\begin{align}
    &\E_{\charseq^{(1)}, \ldots, \charseq^{(M)}\simIID\proposal}\left[\hat{\targetFn} (\charseq)\right] \nonumber \\
    &= \E_{\charseq^{(1)}, \ldots, \charseq^{(M)}\simIID\proposal}\left[\frac{1}{M} \sum_{m=1}^M w^{(m)} \indicator{\charseq = \charseq^{(m)}}\right] &\proofExplain{\cref{def:importance-sampling}}\\
    &=\frac{1}{M} \sum_{m=1}^M \E_{\charseq^{(m)}\simIID\proposal} \left[w^{(m)} \indicator{\charseq = \charseq^{(m)}}\right] &\proofExplain{Independence}\\
    &= \E_{\charseq'\sim\proposal} \left[\frac{\targetFn(\charseq')}{\proposal(\charseq')} \indicator{\charseq = \charseq'} \right]  &\proofExplain{Identicality}\\
    &= \sum_{\charseq'\in\charStrings} \proposal(\charseq') \frac{\targetFn(\charseq')}{\proposal(\charseq')} \indicator{\charseq = \charseq'} &\proofExplain{Def.\@ of \ensuremath{\E}}\\
    &= \sum_{\charseq'\in\charStrings}\targetFn(\charseq') \indicator{\charseq = \charseq'} &\proofExplain{Algebra, Absolute Continuity}\\
    &=\targetFn(\charseq)
\end{align}
\end{subequations}

\textbf{Part (3).}
Pick $\charseq \in \charStrings$ arbitrarily. Since $\E[w] = Z < \infty$, the strong law of large numbers (SLLN) \citep[e.g.,][]{geweke1989bayesian, owen2013importance} gives $\widehat{\targetFn}(\charseq) \to \targetFn(\charseq)$ and $\widehat{Z} \to Z > 0$ almost surely, so $\widehat{\ensemble}(\charseq) \to \ensemble(\charseq)$ almost surely (we set $\widehat{\ensemble} \defeq 0$ when $\widehat{Z} = 0$). As $0 \le \widehat{\ensemble}(\charseq) \le 1$, then bounded convergence gives
\begin{subequations}
\begin{align}
&\lim_{M \to \infty} \E_{\charseq^{(1)}, \ldots, \charseq^{(M)}\simIID\proposal} \left[ \widehat{\ensemble}(\charseq)  \right] \nonumber \\
&= 
\lim_{M \to \infty}
\E_{\charseq^{(1)}, \ldots, \charseq^{(M)}\simIID\proposal} \left[
\frac{\frac{1}{M}\sum_{m=1}^M w^{(m)} \indicator{\charseq = \charseq^{(m)}}}{ \frac{1}{M} \sum_{m=1}^M w^{(m)} }     
\right]   & \proofExplain{\cref{def:importance-sampling}}\\
&= 
\frac{\E_{\charseq'\sim\proposal} \left[\frac{\targetFn(\charseq')}{\proposal(\charseq')} \indicator{\charseq = \charseq'} \right]}{\E_{\charseq' \sim \proposal} \left[ \frac{\targetFn(\charseq')}{\proposal(\charseq')} \right] } 
& \proofExplain{SLLN, bounded convergence} \\
&= 
\frac{ \targetFn(\charseq)}{Z} & \proofExplain{\ensuremath{\hat{Z},~\hat{\targetFn}} unbiased} \\
&= 
\ensemble(\charseq)   & \proofExplain{\cref{def:fensemble}}
\end{align}
\end{subequations}
\end{proof}

\optimalProposal*
\label{prop:optimal-proposal:proof}
\begin{proof}
Suppose we fix a given string $\charseq$, and we wanted to choose the proposal distribution $\proposalPrefix(\cdot \mid \charseq)$ to minimize the variance of the importance weight
\begin{align}
\importanceWeight(\chars', \charseq) \defeq \frac{\shapingPrefix(\chars' \mid \charseq)}{\proposalPrefix(\chars' \mid \charseq)}
\end{align}

The expected value of the importance weight is 
\begin{subequations}
\begin{align}
\E_{\chars' \sim \proposalPrefix(\cdot \mid \charseq)}\left[ \frac{\shapingPrefix(\chars'\mid \charseq)}{\proposalPrefix(\chars' \mid \charseq)}\right] = \sum_{\chars'\in \charAlphabet \cup \{\eos\}} \proposalPrefix(\chars' \mid \charseq) \frac{\shapingPrefix(\chars'\mid \charseq)}{\proposalPrefix(\chars' \mid \charseq)} = \sum_{\chars'\in \charAlphabet \cup \{\eos\}} \shapingPrefix(\chars'\mid \charseq)         \end{align}
\end{subequations}

Thus, the variance of the weight is

\nopagebreak
\begin{subequations}
\begin{align}\label{eq:weight-variance}
\text{Var}\left(\importanceWeight(\chars', \charseq)\right)_{\chars' \sim \proposalPrefix(\cdot \mid \charseq)} &=  \E_{\chars' \sim \proposalPrefix(\cdot \mid \charseq)} \left[ \importanceWeight(\chars', \charseq)^2 \right] - \left(\sum_{\chars'\in \charAlphabet \cup \{\eos\}} \shapingPrefix(\chars'\mid \charseq)\right)^2 \\
&= -\left(\sum_{\chars'\in \charAlphabet \cup \{\eos\}} \shapingPrefix(\chars'\mid \charseq)\right)^2  + \sum_{\chars'\in \charAlphabet \cup \{\eos\}} \proposalPrefix(\chars' \mid \charseq) \frac{\shapingPrefix(\chars' \mid \charseq)^2}{\proposalPrefix(\chars' \mid \charseq)^2} \\
&= -\left(\sum_{\chars'\in \charAlphabet \cup \{\eos\}} \shapingPrefix(\chars'\mid \charseq)\right)^2  + \sum_{\chars' \in \charAlphabet \cup \{\eos\}}  \frac{\shapingPrefix(\chars' \mid \charseq)^2}{\proposalPrefix(\chars' \mid \charseq)}
\end{align}
\end{subequations}

Our goal is to minimize \Cref{eq:weight-variance} under the \emph{constraints} that the proposal should be a valid probability distribution over tokens, i.e., $\proposalPrefix(\chars' \mid \charseq) \ge 0$ for all $\chars'$ and $\sum_{\chars'\in\charAlphabet \cup \{\eos\}} \proposalPrefix(\chars' \mid \charseq) = 1$. We can formulate this problem using the method of Lagrange multipliers, i.e.,
\begin{align}
L 
&= -\left(\sum_{\chars'\in \charAlphabet \cup \{\eos\}} \shapingPrefix(\chars'\mid \charseq)\right)^2  + \sum_{\chars' \in \charAlphabet \cup \{\eos\}}  \frac{\shapingPrefix(\chars' \mid \charseq)^2}{\proposalPrefix(\chars' \mid \charseq)} - \lambda \left(1 - \sum_{\chars' \in \charAlphabet \cup \{\eos\}} \proposalPrefix(\chars' \mid \charseq) \right)
\end{align}
as a function of the Lagrange multiplier $\lambda$. Then,
\begin{subequations}
\begin{align}
\frac{\partial}{\partial \proposalPrefix(\chars'' \mid \charseq)} L 
&= \sum_{\chars'\in \charAlphabet \cup \{\eos\}}  \frac{\partial}{\partial \proposalPrefix(\chars'' \mid \charseq)} \frac{\shapingPrefix(\chars' \mid \charseq)^2}{\proposalPrefix(\chars' \mid \charseq)} + \lambda  \sum_{\chars'\in \charAlphabet \cup \{\eos\}} \frac{\partial}{\partial \proposalPrefix(\chars'' \mid \charseq)} \proposalPrefix(\chars' \mid \charseq)\\
&= \sum_{\chars'\in \charAlphabet \cup \{\eos\}}  - \frac{\shapingPrefix(\chars' \mid \charseq)^2} {\proposalPrefix(\chars' \mid \charseq)^2}\indicator{\chars' = \chars''} + \lambda  \sum_{\chars'\in \charAlphabet \cup \{\eos\}} \indicator{\chars' = \chars''}\\
&= - \frac{\shapingPrefix(\chars'' \mid \charseq)^2} {\proposalPrefix(\chars'' \mid \charseq)^2} + \lambda  
\end{align}
\end{subequations}
Solving for $\frac{\partial L}{\partial \proposalPrefix(\chars'' \mid \charseq)} = 0$ gives 
\begin{subequations}
\begin{align}
0 
&= - \frac{\shapingPrefix(\chars'' \mid \charseq)^2} {\proposalPrefix(\chars'' \mid \charseq)^2} + \lambda \\
\proposalPrefix(\chars'' \mid \charseq)^2 &= \frac{\shapingPrefix(\chars'' \mid \charseq)^2}{\lambda}  \\
\proposalPrefix(\chars'' \mid \charseq) &= \pm \sqrt{\frac{\shapingPrefix(\chars'' \mid \charseq)^2}{\lambda}}
\end{align}
\end{subequations}

As $\proposalPrefix(\chars' \mid \charseq) \ge 0$, the only valid solution is $\proposalPrefix^*(\chars'' \mid \charseq) = \frac{\shapingPrefix(\chars'' \mid \charseq)}{\sqrt{\lambda}}$. 
Substituting the optimal choice into $\proposalPrefix(\chars' \mid \charseq)$ into the constraint $\sum_{\chars'\in\charAlphabet \cup \{\eos\}} \proposalPrefix(\chars' \mid \charseq) = 1$ and solving for the value of $\sqrt{\lambda}$ that makes it hold
\begin{subequations}
\begin{align}
1 &= \sum_{\chars''\in \charAlphabet \cup \{\eos\}} \proposalPrefix^*(\chars'' \mid \charseq) \\
&= 
\sum_{\chars''\in \charAlphabet \cup \{\eos\}} \frac{\shapingPrefix(\chars'' \mid \charseq)}{\sqrt{\lambda}} \\
\sqrt{\lambda} &= 
\sum_{\chars''\in \charAlphabet \cup \{\eos\}} \shapingPrefix(\chars'' \mid \charseq)
\end{align}
\end{subequations}
    
Thus, the optimal choice for $\proposalPrefix(\cdot \mid \charseq)$ is $\proposalPrefix^*(\chars'' \mid \charseq) \propto \shapingPrefix(\chars'' \mid \charseq)$. This choice is also the unique global minimizer, including when some $\shapingPrefix(\chars' \mid \charseq) = 0$. By the Cauchy--Schwarz inequality, $\big(\sum_{\chars'} \shapingPrefix(\chars' \mid \charseq)\big)^2 \le \sum_{\chars'} \shapingPrefix(\chars' \mid \charseq)^2 / \proposalPrefix(\chars' \mid \charseq)$, with equality if and only if $\proposalPrefix(\cdot \mid \charseq) \propto \shapingPrefix(\cdot \mid \charseq)$, so the variance in \cref{eq:weight-variance} is zero exactly at this proposal.
\end{proof}

\subsection{Absolute Continuity and \texorpdfstring{$\func$}{f}}\label{app:absolute-continuity-f}

Given the discussion of absolute continuity, we ask for which $\func$ our standard choice of shaping function, $\shapingPrefix(\charseq) = \func(\pPrefix_{1}(\charseq), \ldots ,\pPrefix_{K}(\charseq))$, automatically satisfies this requirement.

\begin{definition}
For our choice of shaping function (cf. \Cref{sec:sis}), the absolute continuity condition $\shapingPrefix(\charseq) = 0 \Longrightarrow \ensemblePrefix(\xx) = 0$ imposes for $\func$ to fulfill
    \begin{align}
        \func(\pPrefix_{1}(\charseq'), \ldots ,\pPrefix_{K}(\charseq')) = 0 \Longrightarrow \func(\p_{1}(\charseq), \ldots ,\p_{K}(\charseq)) = 0
    \end{align}
for any prefix $\charseq'$ of $\charseq$.  
We call such a set of functions \defn{annihilative}, that is, if any prefix of a string is mapped to zero, then the entire string must also be mapped to zero.
\end{definition}

We note that when $\func$ does not satisfy annihilativity\footnote{An example of non-annihilativity is the class of contrastive ensembles, e.g., $\func(\p_1, \p_2) = \left(\p_1(\charseq) - \p_2(\charseq)\right)^2$, which are sometimes used for language model steering \cite{li-etal-2023-contrastive, dekoninck2024arithmetic}.}; this simply means that the standard shaping function $\shapingPrefix(\charseq) = \func(\pPrefix_{1}(\charseq), \ldots ,\pPrefix_{K}(\charseq))$ is not a suitable choice.
However, simple modifications to our choice of $\shapingPrefix(\charseq)$ can recover absolute continuity.
For instance, setting $\shapingPrefix(\charseq) = |\func(\pPrefix_{1}(\charseq), \ldots ,\pPrefix_{K}(\charseq))| + \epsilon$
for $\epsilon > 0$ may increase the variance of the importance weights, but ensures $\shapingPrefix(\charseq) > 0$ for all $\charseq \in \charStrings$.

\subsection{Alpha-Divergence Ensembles}

\alphaEnsembles*
\begin{proof}
To find the optimal consensus distribution $\ensemble^*$, we minimize the expert-weighted sum of $\alpha$-divergences subject to the constraint that $\ensemble$ lies on the simplex (i.e., $\sum_{\charseq\in \charStrings} \ensemble(\charseq) = 1$). We formulate the Lagrangian $\Lambda$:
\begin{equation}
\begin{aligned}
\Lambda(\ensemble, \nu)
&= \sum_{k=1}^K
\frac{\expertPriorWeight_k}{\alpha(1-\alpha)}
\left(
1 - \sum_{\charseq \in \charStrings}
\ensemble(\charseq)^\alpha
\pExpert_k(\charseq)^{1-\alpha}
\right) \\
&\quad
+ \nu \left( \sum_{\charseq \in \charStrings} \ensemble(\charseq) - 1 \right).
\end{aligned}
\end{equation}
We take the partial derivative with respect to $\ensemble(\charseq)$ for a specific element $\charseq \in \charStrings$:
\begin{equation}
\begin{aligned}
\frac{\partial \Lambda}{\partial \ensemble(\charseq)}
&= -\sum_{k=1}^K \frac{\expertPriorWeight_k}{\alpha(1-\alpha)}
\cdot \alpha \ensemble(\charseq)^{\alpha-1} \pExpert_k(\charseq)^{1-\alpha}
+ \nu \\
&= -\frac{\ensemble(\charseq)^{\alpha-1}}{1-\alpha}
\sum_{k=1}^K \expertPriorWeight_k \pExpert_k(\charseq)^{1-\alpha}
+ \nu .
\end{aligned}
\end{equation}
Setting the derivative to zero for stationarity:
\begin{equation}
\nu = \frac{\ensemble(\charseq)^{\alpha-1}}{1-\alpha} \sum_{k=1}^K \expertPriorWeight_k \pExpert_k(\charseq)^{1-\alpha}.
\end{equation}
Rearranging to solve for $\ensemble(\charseq)$:
\begin{equation}
\begin{aligned}
\ensemble(\charseq)^{\alpha-1}
&= \nu(1-\alpha)
\left(
\sum_{k=1}^K \expertPriorWeight_k
\pExpert_k(\charseq)^{1-\alpha}
\right)^{-1}.
\end{aligned}
\end{equation}
Raising both sides to the power of $\frac{1}{\alpha-1}$ (equivalently $-\frac{1}{1-\alpha}$):
\begin{equation}
\ensemble(\charseq) \propto \left( \sum_{k=1}^K \expertPriorWeight_k \pExpert_k(\charseq)^{1-\alpha} \right)^{\frac{1}{1-\alpha}}.
\end{equation}
Finally, enforcing the normalization constraint $\sum_{\charseq\in \charStrings} \ensemble(\charseq) = 1$ yields the partition function $Z$, recovering the generalized mean form with exponent $\tau = 1-\alpha$. For $\alpha \notin \{0,1\}$, the objective is strictly convex in $\ensemble$ and the constraint is linear, so this stationary point is the unique global minimizer.
\end{proof}

\begin{restatable}{proposition}{annihilativeProp}\label{prop:annihilative-f}
The family of generalized means $M_\familyParam$ defined in \Cref{eq:generalized-means} is annihilative for all $\familyParam \in \mathbb{R} \setminus \{0\}$ and the limit cases $\familyParam \to \pm\infty$, and $\familyParam \to 0$.
\end{restatable}
\begin{proof}
        Let $\pExpert_1(\charseq), \ldots, \pExpert_K(\charseq)$ denote the probabilities assigned to string $\charseq$ by $K$ language models, and let $\charseq'$ be a prefix of $\charseq$. We denote the prefix probabilities by $\prefixOp{\pExpert}_1(\charseq'), \ldots, \prefixOp{\pExpert}_K(\charseq')$. By \cref{eq:prefix-prob}, $\pExpert_i(\charseq) \le \prefixOp{\pExpert}_i(\charseq')$, so $\prefixOp{\pExpert}_i(\charseq') = 0$ implies $\pExpert_i(\charseq) = 0$. We assume $\expertPriorWeight_i > 0$ for all $i$. We evaluate different domains of $\familyParam$: 
    \begin{enumerate}
        \item $\familyParam > 0$. For positive $\familyParam$, we have
        \begin{equation}
            M_\familyParam(\prefixOp{\pExpert}_1(\charseq'), \ldots, \prefixOp{\pExpert}_K(\charseq')) = \left(\sum_{i=1}^K \expertPriorWeight_i \prefixOp{\pExpert}_i(\charseq')^\familyParam\right)^{1/\familyParam} = 0
        \end{equation}
        if and only if $\sum_{i=1}^K \prefixOp{\pExpert}_i(\charseq')^\familyParam = 0$. Since each term $\prefixOp{\pExpert}_i(\charseq')^\familyParam \geq 0$, this requires $\prefixOp{\pExpert}_i(\charseq') = 0$ for all $i \in \{1, \ldots, K\}$.  By this bound, this implies $\pExpert_k(\charseq) = 0$, which implies
        \begin{equation}
            M_\familyParam(\pExpert_1(\charseq), \ldots, \pExpert_K(\charseq)) = 0.
        \end{equation}
        Thus, annihilativity holds.
        \item $\familyParam \to 0$. Taking the limit as $\familyParam \to 0$, we obtain the geometric mean:
        \begin{equation}
            M_0(\prefixOp{\pExpert}_1(\charseq'), \ldots, \prefixOp{\pExpert}_K(\charseq')) = \prod_{i=1}^K \prefixOp{\pExpert}_i(\charseq')^{\expertPriorWeight_i} = 0
        \end{equation}
        if and only if $\prefixOp{\pExpert}_i(\charseq') = 0$ for at least one $i \in \{1, \ldots, K\}$. By this bound, if $\prefixOp{\pExpert}_i(\charseq') = 0$, then $\pExpert_i(\charseq) = 0$, which implies
        \begin{equation}
            M_0(\pExpert_1(\charseq), \ldots, \pExpert_K(\charseq)) = \prod_{j=1}^K \pExpert_j(\charseq)^{\expertPriorWeight_j} = 0.
        \end{equation}
        Thus, annihilativity holds.
        \item  $\familyParam < 0$. For negative $\familyParam$, the generalized mean is defined by continuity: if any $\prefixOp{\pExpert}_i(\charseq') = 0$, we set $M_\familyParam = 0$. If $M_\familyParam(\prefixOp{\pExpert}_1(\charseq'), \ldots, \prefixOp{\pExpert}_K(\charseq')) = 0$, then at least one $\prefixOp{\pExpert}_i(\charseq') = 0$. By this bound, $\pExpert_i(\charseq) = 0$, which implies
        \begin{equation}
            M_\familyParam(\pExpert_1(\charseq), \ldots, \pExpert_K(\charseq)) = 0.
        \end{equation}
        Thus, annihilativity holds. 
        \item $\familyParam \to \pm\infty$. In the limit $\familyParam \to -\infty$, we obtain $M_{-\infty} = \min_{i} \prefixOp{\pExpert}_i(\charseq')$, which equals zero if and only if at least one $\prefixOp{\pExpert}_i(\charseq') = 0$. In the limit $\familyParam \to +\infty$, we obtain $M_{+\infty} = \max_{i} \prefixOp{\pExpert}_i(\charseq')$, which equals zero if and only if all $\prefixOp{\pExpert}_i(\charseq') = 0$. In both cases, by this bound, the corresponding $\pExpert_i(\charseq) = 0$, which implies $M_{\pm\infty}(\pExpert_1(\charseq), \ldots, \pExpert_K(\charseq)) = 0$. Thus, annihilativity holds.
    \end{enumerate}
\end{proof}

\section{Experimental Details}\label{sec:exp-details}
We served models on two 24GB RTX 4090 GPUs and ran the SMC loop on one 24GB RTX 3090 GPU, where the generated predictions were aggregated. We used vLLM\footnote{\url{https://docs.vllm.ai/en/}} for serving all models. The models are loaded in bfloat16 data type. The sampling temperature is generally set to $1.0$, except for GSM8K, where a lower value of $0.7$ is required to approximately match baseline performance.

To run thorough multi-seed ablations, for each dataset, we evaluate 100 randomly sampled instances.
For byte-level ensembles, we use a beam size of 20 and a pruning threshold of 0.005.

\myparagraph{Runtime analysis.} In \autoref{tab:runtimes_json}, we show per sample, averaged runtimes for token-level, as well as byte-level SMC, and compare it to running inference on a single base model, and doing standard locally normalized ensembling. Local ensembling incurs approximately 6-8x overhead over base model inference, reflecting the cost of running two forward passes. Token-level SMC with $M=10$ particles requires 4.6-4.9 seconds per instance, roughly 2x slower than local ensembling, with runtime scaling approximately linearly in $M$. Byte-level SMC is substantially more expensive, requiring 76-155 seconds per instance at $
M=10$, well above token-level SMC with $
M=75$. This increased cost reflects the longer sequence lengths at the byte level, since each token expands to multiple bytes, which increases the number of SMC steps and the number of model calls required to score each particle.
A way to improve runtimes is to enhance batching or optimize pruning in our byte-level language model.

\begin{table}[ht]
\small
\centering
\caption{Average per-instance runtime (seconds) on JSON. Token SMC uses $M$ particles.}
\label{tab:runtimes_json}
\begin{tabular}{l r rrrr}
\toprule
& & &\multicolumn{3}{c}{Token SMC} \\
\cmidrule(lr){4-6}
Model & Base Model & Local Ensemble & $M{=}10$ & $M{=}25$ & $M{=}75$ \\
\midrule
Llama & 0.33 & 1.93 & 4.62 & 8.62 & 19.54 \\
Qwen & 0.27 & 2.19 & 4.86 & 9.14 & 26.24 \\
\bottomrule
\end{tabular}
\end{table}

\section{Asset Licensing}\label{sec:licensing}

\paragraph{Models.}
\begin{itemize}
    \item Llama-3.1-8B-Instruct \citep{grattafiori2024llama3herdmodels}: Llama 3.1 Community License.
    \item Qwen2.5-7B-Instruct \citep{qwen2025qwen25technicalreport}: Apache 2.0.
    \item Phi-4 \citep{abdin2024phi4technicalreport}: MIT.
\end{itemize}

\paragraph{Datasets.}
\begin{itemize}
    \item MMLU \citep{hendrycks2021mmlu}: MIT
    \item GSM8K \citep{cobbe2021gsm8k}: MIT
    \item Spider \citep{yu-etal-2018-spider}: CC BY-SA 4.0
    \item BBH \citep{suzgun-etal-2023-challenging}: MIT
    \item JSONSchemaBench \citep{geng2025json}: MIT
\end{itemize}

\section{Additional Results}\label{sec:appendix-additional-results}

\subsection{Detailed Ensembling Results}
Complementing the condensed results shown in \cref{fig:within-model,fig:f-ensemble-comparisons}, we provide a version aggregated by ensembling method in \cref{tab:func-comparison}, which provides numerical detail behind \cref{fig:f-ensemble-diffs-heatmap}, along with the complete set of results in \cref{tab:main-results}.

\begin{table*}
\centering
\caption{Comparison of $\func$-ensemble functions (token-level SMC). \textbf{Bold}: highest average expected accuracy in column, unless a single prompt is higher. $^{\dagger}$: non-overlapping 95\% CIs with best baseline. \colorbox{gray!15}{Shaded}: strongest single-prompt baseline.}
\label{tab:func-comparison}
\resizebox{\linewidth}{!}{%
\begin{tabular}{@{}l ccc ccc ccc@{}}
\toprule
& \multicolumn{3}{c}{BBH (Word Sorting)} & \multicolumn{3}{c}{SPIDER (Text-to-SQL)} & \multicolumn{3}{c}{JSON Schema} \\
\cmidrule(lr){2-4} \cmidrule(lr){5-7} \cmidrule(lr){8-10}
\textbf{Method} & Llama & Qwen & Phi & Llama & Qwen & Phi & Llama & Qwen & Phi \\
\midrule
Prompt 1 & \cellcolor{gray!15}\res{37.1}{0.7} & \cellcolor{gray!15}\res{13.6}{0.4} & \cellcolor{gray!15}\res{26.2}{0.4} & \res{36.8}{0.2} & \cellcolor{gray!15}\res{54.5}{0.3} & \cellcolor{gray!15}\res{36.4}{0.1} & \res{83.3}{0.2} & \cellcolor{gray!15}\res{95.9}{0.2} & \res{82.3}{0.7} \\
Prompt 2 & \res{35.9}{0.3} & \res{4.5}{0.3} & \res{25.7}{0.4} & \cellcolor{gray!15}\res{36.9}{0.4} & \res{50.2}{0.5} & \res{30.9}{0.2} & \cellcolor{gray!15}\res{86.1}{0.7} & \res{94.3}{0.1} & \cellcolor{gray!15}\res{87.8}{0.4} \\
\midrule
$\min$ & \win{\res{43.2}{1.3}}\sig & \win{\res{16.0}{1.6}}\sig & \win{\res{29.0}{1.3}}\sig & \win{\res{42.4}{1.4}}\sig & \win{\res{56.4}{1.1}}\sig & \res{35.8}{1.5} & \win{\res{89.7}{0.9}}\sig & \res{95.6}{0.6} & \win{\res{88.3}{0.4}} \\
Product & \res{41.2}{0.9}\sig & \res{15.8}{1.4}\sig & \res{28.5}{1.1}\sig & \res{41.9}{0.8}\sig & \res{56.2}{1.2}\sig & \res{35.7}{1.8} & \res{88.9}{0.9}\sig & \res{95.5}{0.4} & \res{87.1}{0.6} \\
Mixture & \res{37.5}{1.0} & \res{9.6}{0.6} & \res{27.3}{1.0} & \res{39.3}{1.4}\sig & \res{52.4}{1.7} & \res{33.6}{0.7} & \res{86.9}{1.3} & \res{95.5}{0.6} & \res{86.1}{0.4} \\
$\max$ & \res{35.3}{0.8} & \res{9.1}{0.5} & \res{26.1}{1.0} & \res{38.1}{0.9}\sig & \res{51.8}{1.4} & \res{32.6}{0.4} & \res{85.1}{2.0} & \res{95.2}{0.7} & \res{85.5}{1.1} \\
\bottomrule
\end{tabular}%
}
\end{table*}

\begin{table*}
\centering
\caption{Comparison of $\func$-ensemble functions (token-level SMC) for GSM8K and MMLU. \textbf{Bold}: highest average expected accuracy in column, unless a single prompt is higher. $^{\dagger}$: non-overlapping 95\% CIs with best baseline. \colorbox{gray!15}{Shaded}: strongest single-prompt baseline.}
\label{tab:gsm8k-mmlu-comparison}
\resizebox{0.7\linewidth}{!}{%
\begin{tabular}{@{}l ccc ccc@{}}
\toprule
& \multicolumn{3}{c}{GSM8K} & \multicolumn{3}{c}{MMLU} \\
\cmidrule(lr){2-4} \cmidrule(lr){5-7}
\textbf{Method} & Llama & Qwen & Phi & Llama & Qwen & Phi \\
\midrule
Prompt 1 & \res{72.1}{0.4} & \res{66.8}{0.6} & \cellcolor{gray!15}\res{15.4}{0.4} & \res{45.5}{1.0} & \res{71.9}{0.3} & \res{73.5}{0.1} \\
Prompt 2 & \cellcolor{gray!15}\res{73.6}{0.6} & \cellcolor{gray!15}\res{84.8}{0.4} & \res{5.3}{0.2} & \cellcolor{gray!15}\res{58.3}{0.2} & \cellcolor{gray!15}\res{73.3}{0.2} & \cellcolor{gray!15}\res{77.2}{0.5} \\
\midrule
$\min$ & \res{76.5}{3.5} & \res{79.1}{4.0} & \res{28.4}{2.5}\sig & \res{57.2}{1.1} & \res{72.8}{0.3} & \res{75.2}{1.1} \\
Product & \win{\res{78.1}{2.8}}\sig & \res{77.2}{3.4} & \win{\res{33.1}{2.9}}\sig & \res{54.5}{1.6} & \res{72.8}{0.6} & \res{75.3}{1.0} \\
Mixture & \res{73.4}{1.6} & \res{75.0}{1.8} & \res{3.9}{0.9} & \res{49.7}{1.5} & \res{71.9}{1.1} & \res{74.2}{0.7} \\
$\max$ & \res{73.4}{1.6} & \res{74.8}{1.6} & \res{3.8}{0.8} & \res{46.8}{1.4} & \res{71.6}{1.0} & \res{73.6}{0.5} \\
\bottomrule
\end{tabular}%
}
\end{table*}

\begin{table*}
\centering
\caption{Table reporting detailed expected ensemble and base model accuracy at equal model weights, 10 particles with 95\% confidence intervals evaluated across 5 seeds. For cross-model ensembling, both models use Prompt 1.}
\label{tab:main-results}
\footnotesize
\setlength{\tabcolsep}{1.9pt}

\newcommand{\crosslabel}[1]{%
  \sbox0{\shortstack[l]{#1 \\ Cross-Model Byte SMC}}%
  \raisebox{\dimexpr-12pt-0.5\ht0+0.5\dp0\relax}{\usebox0}%
}

\begin{tabular}{l ccc ccc ccc}
\toprule
$\func$-Ensemble & \multicolumn{3}{c}{Big-Bench Hard (Word Sorting)} & \multicolumn{3}{c}{SPIDER (Text-to-SQL)} & \multicolumn{3}{c}{JSON Schema} \\
\cmidrule(lr){2-4} \cmidrule(lr){5-7} \cmidrule(lr){8-10}
& Llama & Qwen & Phi & Llama & Qwen & Phi & Llama & Qwen & Phi \\
\midrule \midrule
Prompt 1 & \res{37.1}{0.7} & \res{13.6}{0.4} & \res{26.2}{0.4} & \res{36.8}{0.2} & \res{54.5}{0.3} & \res{36.4}{0.1} & \res{83.3}{0.2} & \res{95.9}{0.2} & \res{82.3}{0.7} \\
Prompt 2 & \res{35.9}{0.3} & \res{4.5}{0.3} & \res{25.7}{0.4} & \res{36.9}{0.4} & \res{50.2}{0.5} & \res{30.9}{0.2} & \res{86.1}{0.7} & \res{94.3}{0.1} & \res{87.8}{0.4} \\
\midrule \midrule
$\min$, Token Local & \res{38.8}{2.0} & \res{16.0}{2.5} & \res{26.6}{3.6} & \res{38.2}{4.8} & \res{53.4}{1.4} & \res{35.8}{2.8} & \res{87.0}{2.3} & \res{95.6}{1.4} & \res{87.8}{4.2} \\
$\min$, Token SMC & \res{43.2}{1.3} & \res{16.0}{1.6} & \res{29.0}{1.3} & \res{42.4}{1.4} & \res{56.4}{1.1} & \res{35.8}{1.5} & \res{89.7}{0.9} & \res{95.6}{0.6} & \res{88.3}{0.4} \\
\greydashedrule
\crosslabel{$\min$,} & \tikzmark{minL1} & \tikzmark{minQ1} & \tikzmark{minP1} & \tikzmark{minL2} & \tikzmark{minQ2} & \tikzmark{minP2} & \tikzmark{minL3} & \tikzmark{minQ3} & \tikzmark{minP3} \\[26pt]
\midrule
Product, Token Local & \res{36.8}{3.0} & \res{16.0}{1.8} & \res{26.2}{1.6} & \res{38.4}{2.1} & \res{53.4}{1.9} & \res{35.0}{2.6} & \res{87.4}{3.4} & \res{95.6}{2.4} & \res{87.2}{4.2} \\
Product, Token SMC & \res{41.2}{0.9} & \res{15.8}{1.4} & \res{28.5}{1.1} & \res{41.9}{0.8} & \res{56.2}{1.2} & \res{35.7}{1.8} & \res{88.9}{0.9} & \res{95.5}{0.4} & \res{87.1}{0.6} \\
\greydashedrule
\crosslabel{Product,} & \tikzmark{prodL1} & \tikzmark{prodQ1} & \tikzmark{prodP1} & \tikzmark{prodL2} & \tikzmark{prodQ2} & \tikzmark{prodP2} & \tikzmark{prodL3} & \tikzmark{prodQ3} & \tikzmark{prodP3} \\[26pt]
\midrule
Mixture, Token Local & \res{34.6}{2.9} & \res{10.8}{2.8} & \res{26.0}{0.9} & \res{37.4}{3.8} & \res{53.4}{2.9} & \res{34.0}{4.2} & \res{87.0}{3.8} & \res{95.4}{2.4} & \res{86.2}{3.2} \\
Mixture, Token SMC & \res{37.5}{1.0} & \res{9.6}{0.6} & \res{27.3}{1.0} & \res{39.3}{1.4} & \res{52.4}{1.7} & \res{33.6}{0.7} & \res{86.9}{1.3} & \res{95.5}{0.6} & \res{86.1}{0.4} \\
\greydashedrule
\crosslabel{Mixture,} & \tikzmark{mixL1} & \tikzmark{mixQ1} & \tikzmark{mixP1} & \tikzmark{mixL2} & \tikzmark{mixQ2} & \tikzmark{mixP2} & \tikzmark{mixL3} & \tikzmark{mixQ3} & \tikzmark{mixP3} \\[26pt]
\midrule
$\max$, Token Local & \res{32.6}{2.3} & \res{10.6}{3.1} & \res{26.0}{1.8} & \res{37.8}{2.8} & \res{52.8}{3.0} & \res{31.8}{3.2} & \res{85.0}{5.9} & \res{95.4}{1.7} & \res{85.2}{3.0} \\
$\max$, Token SMC & \res{35.3}{0.8} & \res{9.1}{0.5} & \res{26.1}{1.0} & \res{38.1}{0.9} & \res{51.8}{1.4} & \res{32.6}{0.4} & \res{85.1}{2.0} & \res{95.2}{0.7} & \res{85.5}{1.1} \\
\crosslabel{$\max$,} & \tikzmark{maxL1} & \tikzmark{maxQ1} & \tikzmark{maxP1} & \tikzmark{maxL2} & \tikzmark{maxQ2} & \tikzmark{maxP2} & \tikzmark{maxL3} & \tikzmark{maxQ3} & \tikzmark{maxP3} \\[26pt]
\bottomrule
\end{tabular}

\begin{tikzpicture}[remember picture, overlay]
\footnotesize
\tikzset{
  crossline/.style={black, line width=0.2pt},
  endpoint/.style={circle, fill=gray, inner sep=0pt, minimum size=2.5pt}
}
\newcommand{\drawcross}[6]{
  \draw[crossline] ([yshift=-3pt]pic cs:#1) -- ([yshift=-3pt]pic cs:#2)
    node[midway, fill=white, inner sep=1pt] {#4};
  \node[endpoint] at ([yshift=-3pt]pic cs:#1) {};
  \node[endpoint] at ([yshift=-3pt]pic cs:#2) {};
  \draw[crossline] ([yshift=-12pt]pic cs:#2) -- ([yshift=-12pt]pic cs:#3)
    node[midway, fill=white, inner sep=1pt] {#5};
  \node[endpoint] at ([yshift=-12pt]pic cs:#2) {};
  \node[endpoint] at ([yshift=-12pt]pic cs:#3) {};
  \draw[crossline] ([yshift=-21pt]pic cs:#1) -- ([yshift=-21pt]pic cs:#3)
    node[midway, fill=white, inner sep=1pt] {#6};
  \node[endpoint] at ([yshift=-21pt]pic cs:#1) {};
  \node[endpoint] at ([yshift=-21pt]pic cs:#3) {};
}

\drawcross{minL1}{minQ1}{minP1}{\res{25.8}{2.2}}{\res{28.4}{1.1}}{\res{42.2}{1.6}}
\drawcross{minL2}{minQ2}{minP2}{\res{53.4}{2.2}}{\res{52.2}{1.8}}{\res{39.7}{2.3}}
\drawcross{minL3}{minQ3}{minP3}{\res{94.8}{1.1}}{\res{93.0}{1.1}}{\res{88.0}{1.6}}

\drawcross{prodL1}{prodQ1}{prodP1}{\res{33.3}{1.6}}{\res{29.6}{2.0}}{\res{44.2}{1.7}}
\drawcross{prodL2}{prodQ2}{prodP2}{\res{56.8}{1.7}}{\res{53.9}{2.1}}{\res{40.1}{1.2}}
\drawcross{prodL3}{prodQ3}{prodP3}{\res{95.5}{0.5}}{\res{94.5}{0.6}}{\res{87.8}{0.5}}

\drawcross{mixL1}{mixQ1}{mixP1}{\res{27.3}{1.5}}{\res{20.3}{0.6}}{\res{33.6}{1.0}}
\drawcross{mixL2}{mixQ2}{mixP2}{\res{47.4}{1.5}}{\res{45.8}{1.0}}{\res{38.2}{1.2}}
\drawcross{mixL3}{mixQ3}{mixP3}{\res{91.9}{1.2}}{\res{90.0}{0.8}}{\res{85.8}{1.2}}

\drawcross{maxL1}{maxQ1}{maxP1}{\res{26.8}{1.7}}{\res{19.8}{0.2}}{\res{31.5}{0.6}}
\drawcross{maxL2}{maxQ2}{maxP2}{\res{46.8}{2.0}}{\res{45.1}{0.6}}{\res{37.5}{1.5}}
\drawcross{maxL3}{maxQ3}{maxP3}{\res{92.3}{1.1}}{\res{89.9}{1.1}}{\res{86.1}{1.2}}
\end{tikzpicture}
\end{table*}

\subsection{Prompt Performance Overlap and Resulting Ensemble Performance}\label{sec:app-overlap}
Complementing our discussion in \Cref{sec:results}, we examine how base model agreement affects ensemble performance. \Cref{fig:base-vs-ensemble} plots per-sample expected accuracy under each prompt, with color indicating product ensemble accuracy and marker opacity distinguishing samples where the ensemble exceeds the best single-prompt baseline.

The results show that ensembles primarily improve performance on samples where prompts perform moderately (the diagonal), rather than those where one prompt dominates. This is consistent with consensus-seeking. When experts disagree, the product ensemble's gains are limited, whereas when both partially succeed, the ensemble concentrates probability on their intersection.

The pattern varies across models. For Llama, many samples lie along the diagonal, where ensembling may help. In contrast, for Qwen, samples tend to cluster near the axes (one prompt succeeds, the other fails) or in corners (both succeed or both fail), leaving fewer samples on the diagonal. This may explain why Qwen shows smaller ensemble improvements in \autoref{tab:main-results}.

\begin{figure*}[!ht]
    \centering
    \includegraphics[width=\linewidth]{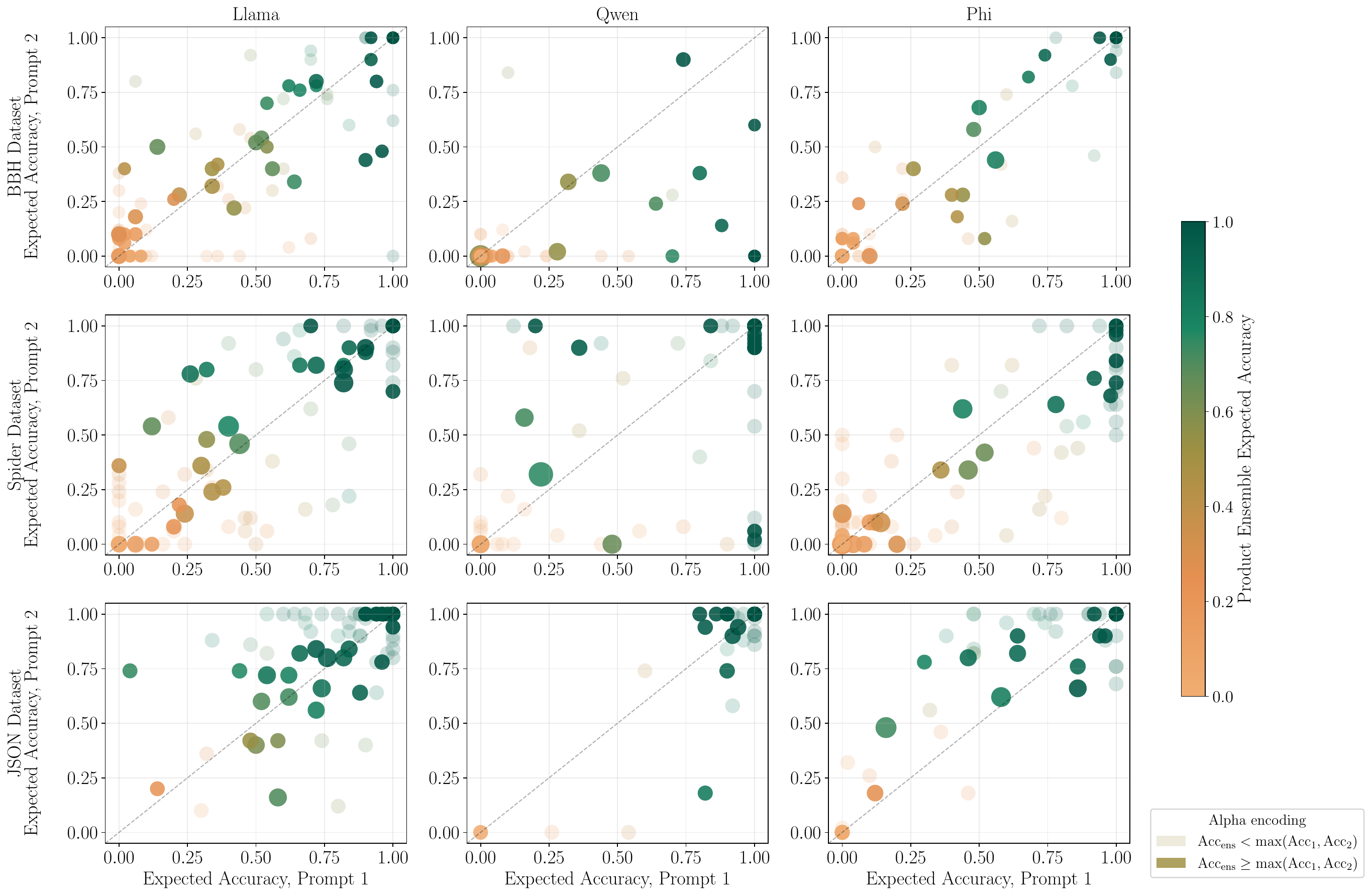}
    \caption{
    Scatter plots comparing expected accuracy of two prompt templates for all datasets and models. Color indicates product ensemble (token-level SMC) expected accuracy. Points with higher alpha indicate cases where the ensemble exceeds the best single template. Marker size scales with the improvement margin.}
    \label{fig:base-vs-ensemble}
\end{figure*}

\subsection{Particle Study}\label{sec:app-particle-study}

Finally, in \autoref{fig:particle-study-json}, \autoref{fig:particle-study-bbh} and \autoref{fig:particle-study-spider}, we plot the approximation quality as a function of the number of particles. As expected, as the number of particles increases, the approximation quality improves and variance reduces.
Visually, we see that the increase in approximation quality saturates at around 10 to 25 particles for the evaluated scenarios.
Following prior work by \citet{loula2025syntactic}, and to balance inference time with approximation cost, we choose to evaluate all setups using 10 particles.

\Cref{fig:particle-study-eacc} then shows expected accuracy as a function of the number of particles. For $\min$ and product ensembles, expected accuracy tends to increase over local ensembling ($M=1$) within the first few particles and remains stable beyond $M=10$. For sum and $\max$ ensembles, it stays close to the average of the two prompts for all $M$.

\begin{figure}
    \centering
    \includegraphics[width=0.75\linewidth]{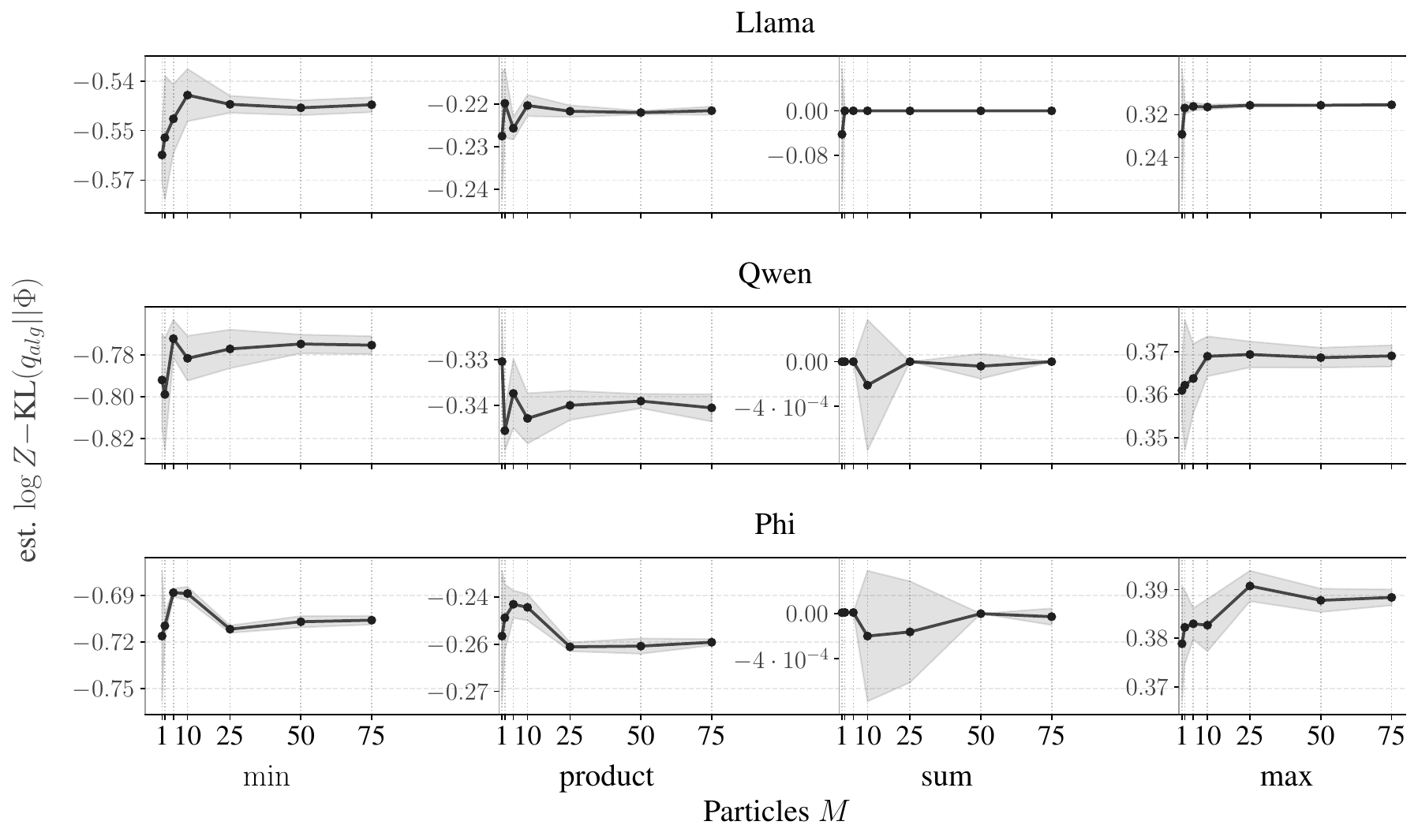}
    \caption{Particle study for JSON. We evaluate the approximation quality estimate $\mathbb{E}[\log \widehat Z]$ as a function of the number of particles being used for SMC.}
    \label{fig:particle-study-json}
\end{figure}

\begin{figure}
    \centering
    \includegraphics[width=0.75\linewidth]{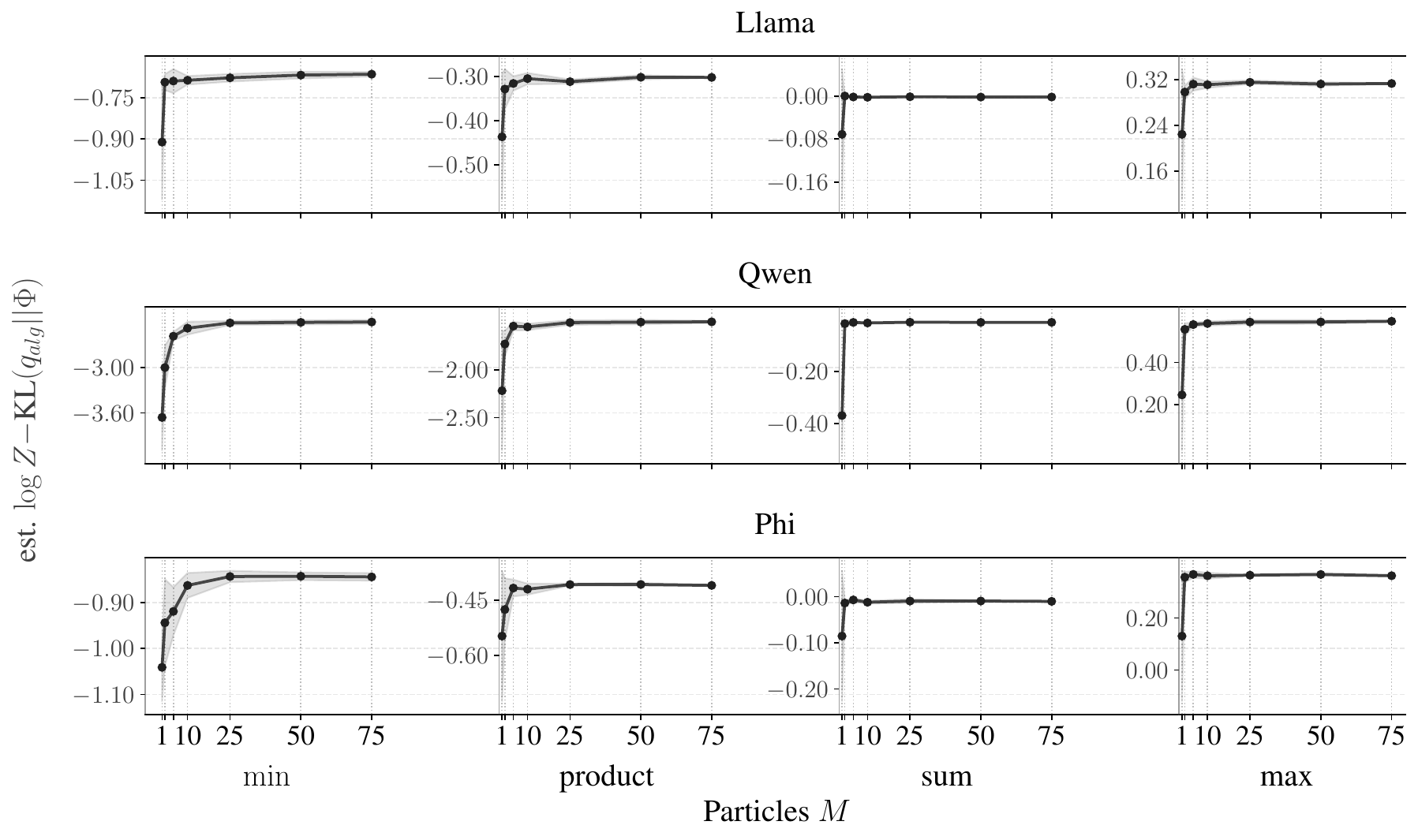}
    \caption{Particle study for BBH. We evaluate the approximation quality estimate $\mathbb{E}[\log \widehat Z]$ as a function of the number of particles being used for SMC.}
    \label{fig:particle-study-bbh}
\end{figure}

\begin{figure}
    \centering
    \includegraphics[width=0.75\linewidth]{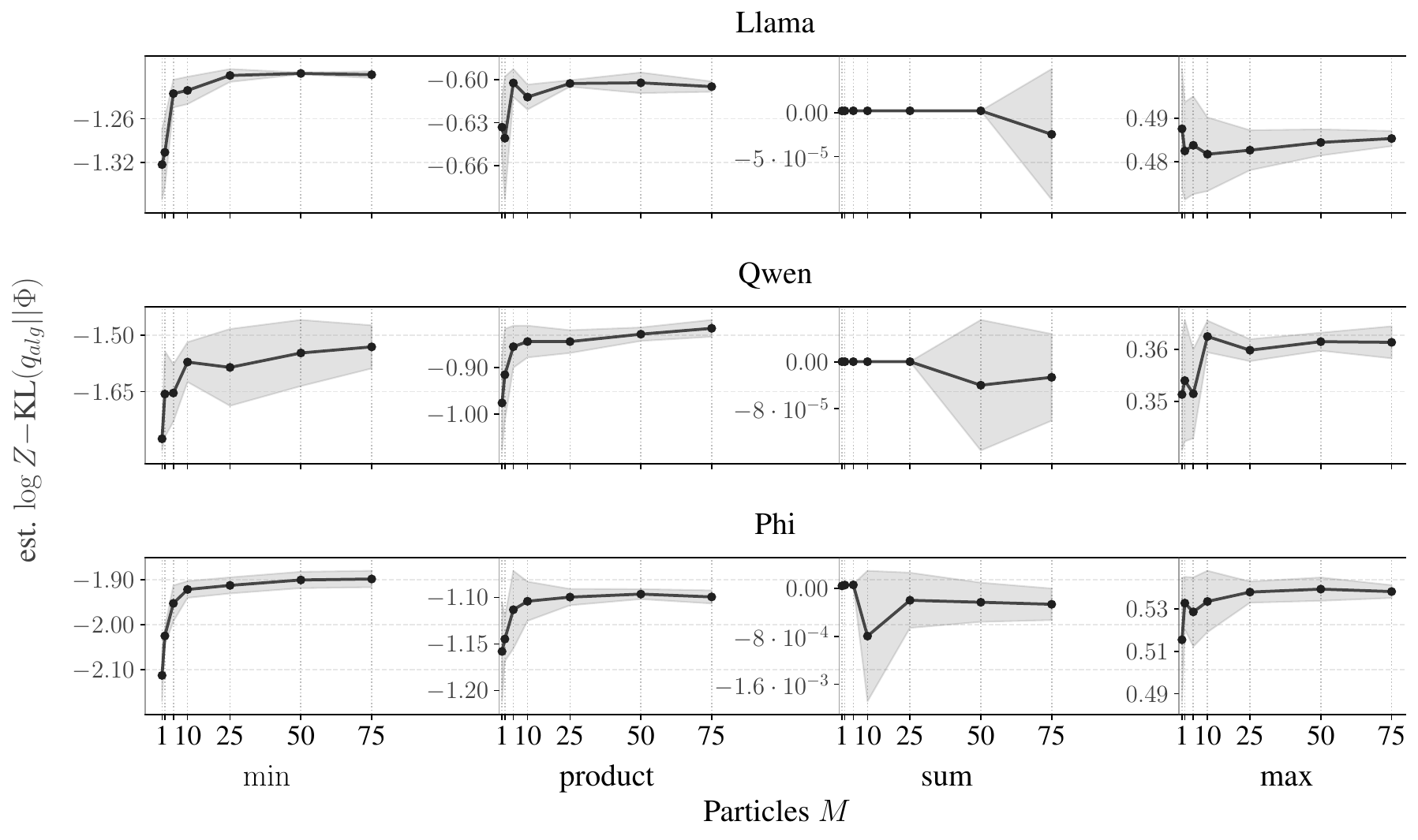}
    \caption{Particle study for SPIDER. We evaluate the approximation quality estimate $\mathbb{E}[\log \widehat Z]$ as a function of the number of particles being used for SMC.}
    \label{fig:particle-study-spider}
\end{figure}

\begin{figure}
    \centering
    \includegraphics[width=\linewidth]{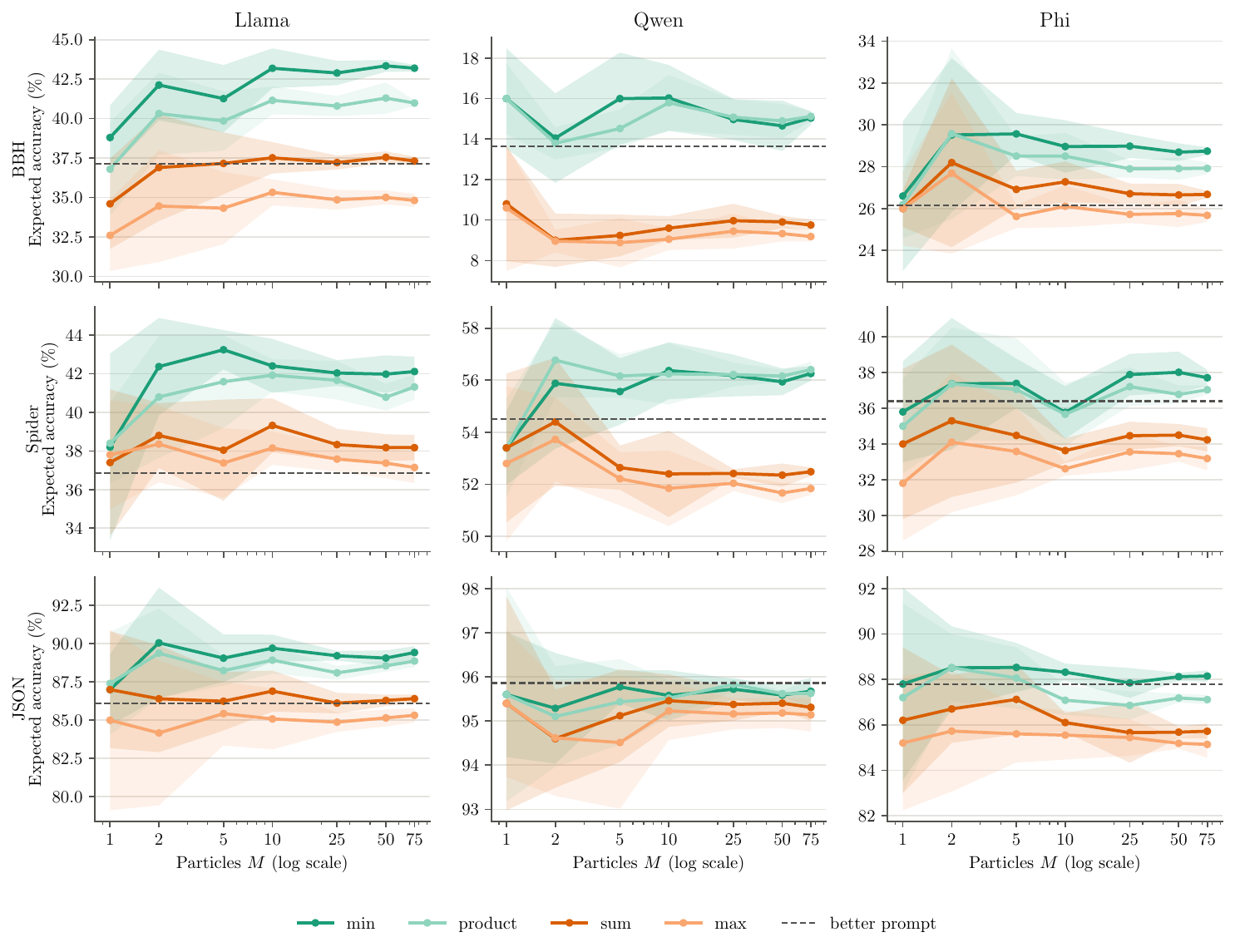}
    \caption{Expected accuracy as a function of the number of particles $M$ (log scale) for within-model $\func$-ensembles (token-level SMC) with 95\% CIs across 5 seeds. $M{=}1$ corresponds to local ensembling. The dashed line marks the better single prompt.}
    \label{fig:particle-study-eacc}
\end{figure}

\subsection{Additional Correlation Computations}
Complementing \autoref{fig:approximation-quality-bbh}, we provide complete plots for all models and datasets in \autoref{fig:approx-full-json}, \autoref{fig:approx-full-bbh} and \autoref{fig:approx-full-spider}.

\begin{figure*}
    \centering
    \includegraphics[width=0.95\linewidth]{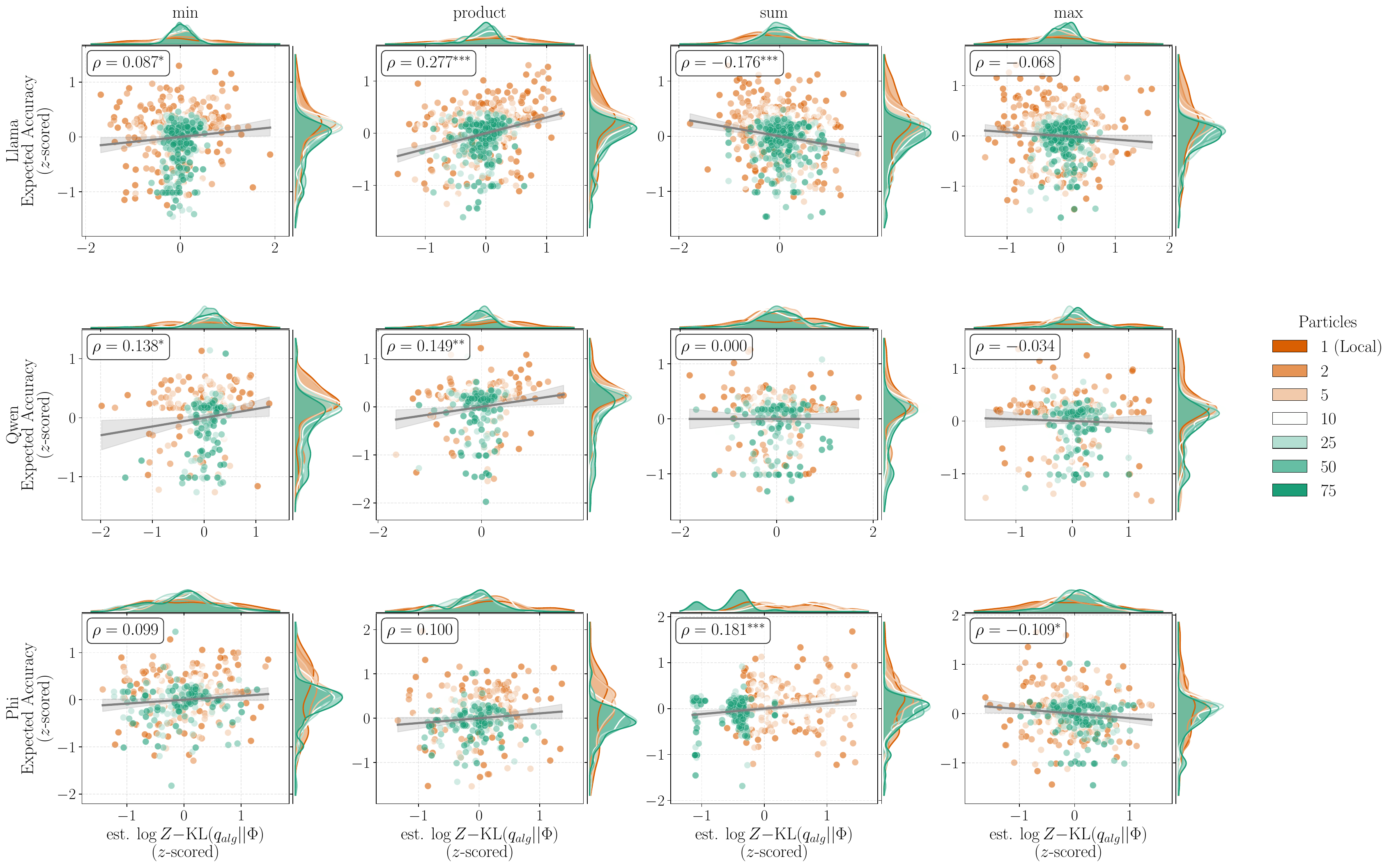}
    \caption{Relationship between expected accuracy and the estimated log marginal likelihood, $\log \hat Z$, on JSON for all models. Each datapoint represents (sample, particle) configurations aggregated across seeds. Both metrics are $z$-scored per example to normalize for variations in problem difficulty and scale. Points are colored by the number of particles used for estimation. Marginal plots display the kernel density estimates for the respective axes. The grey line and shaded band indicate a linear regression fit with a 95\% confidence interval. Pearson correlation coefficients $\rho$ are annotated in the top-left ($^{***}p<0.001$, $^{**}p<0.01$, $^{*}p<0.05$).}
    \label{fig:approx-full-json}
\end{figure*}

\begin{figure*}
    \centering
    \includegraphics[width=0.95\linewidth]{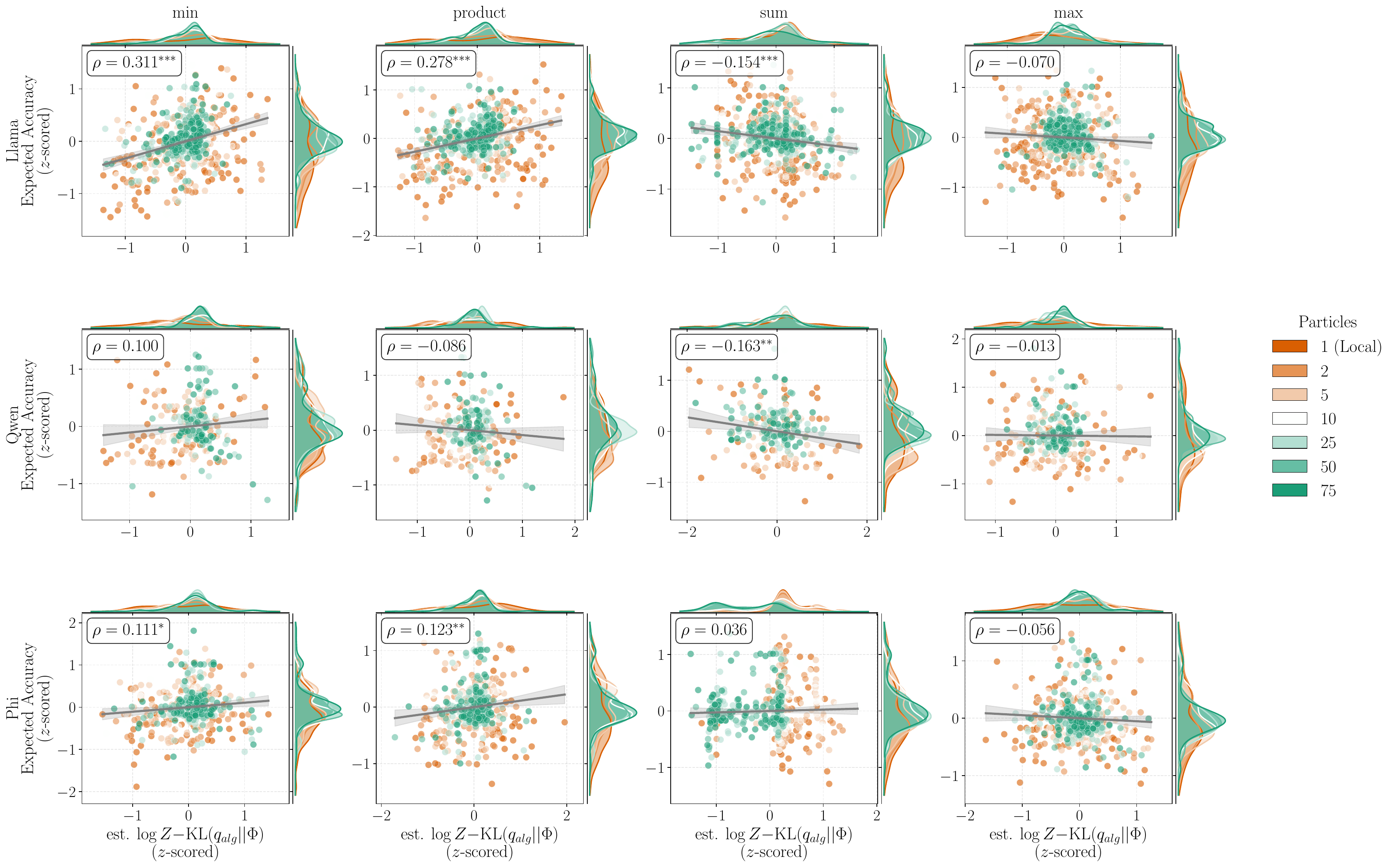}
    \caption{Relationship between expected accuracy and the estimated log marginal likelihood, $\log \hat Z$, on BBH for all models. Each datapoint represents (sample, particle) configurations aggregated across seeds. Both metrics are $z$-scored per example to normalize for variations in problem difficulty and scale. Points are colored by the number of particles used for estimation. Marginal plots display the kernel density estimates for the respective axes. The grey line and shaded band indicate a linear regression fit with a 95\% confidence interval. Pearson correlation coefficients $\rho$ are annotated in the top-left ($^{***}p<0.001$, $^{**}p<0.01$, $^{*}p<0.05$).}
    \label{fig:approx-full-bbh}
\end{figure*}

\begin{figure*}
    \centering
    \includegraphics[width=0.95\linewidth]{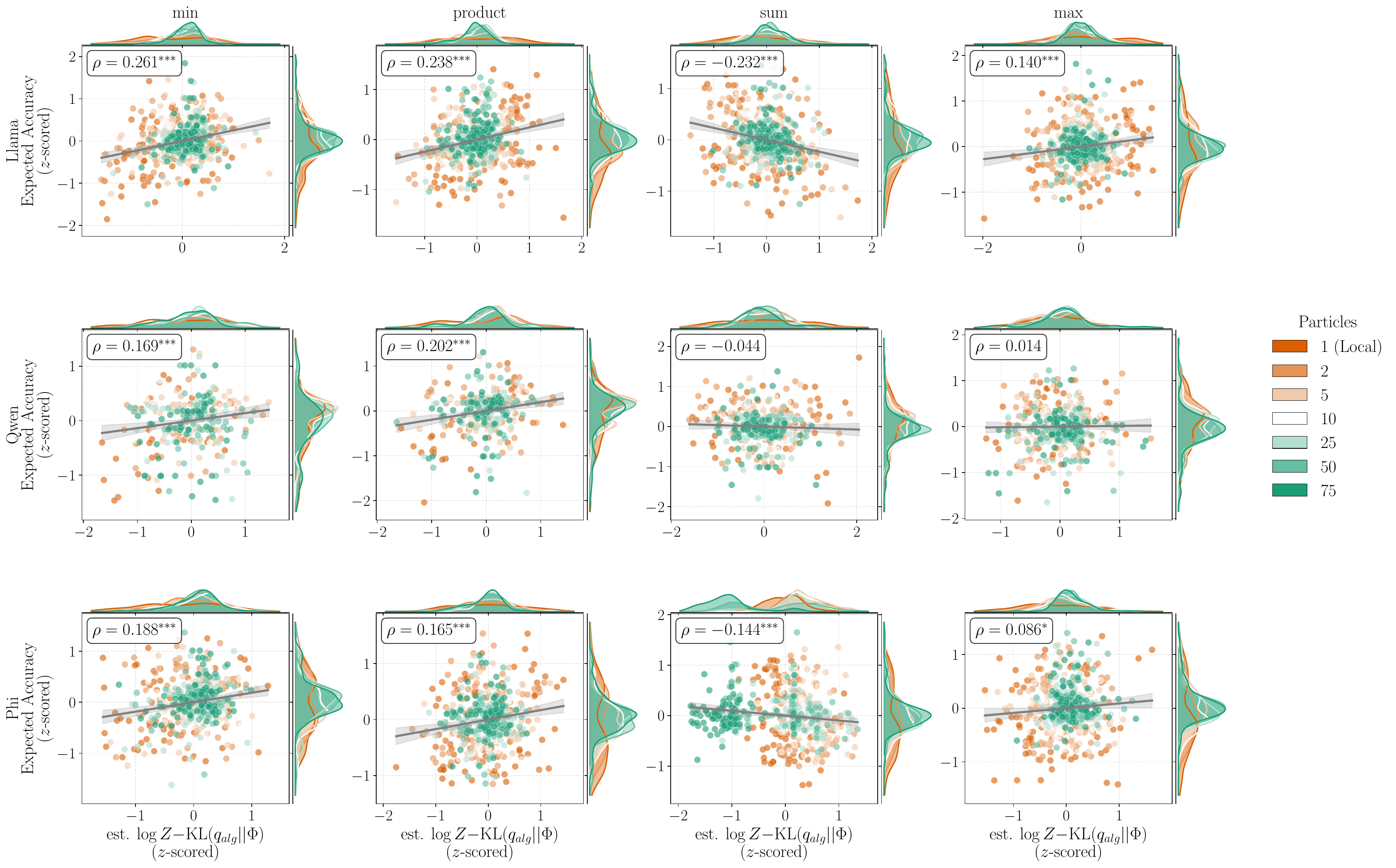}
    \caption{Relationship between expected accuracy and the estimated log marginal likelihood, $\log \hat Z$, on SPIDER for all models. Each datapoint represents (sample, particle) configurations aggregated across seeds. Both metrics are $z$-scored per example to normalize for variations in problem difficulty and scale. Points are colored by the number of particles used for estimation. Marginal plots display the kernel density estimates for the respective axes. The grey line and shaded band indicate a linear regression fit with a 95\% confidence interval. Pearson correlation coefficients $\rho$ are annotated in the top-left ($^{***}p<0.001$, $^{**}p<0.01$, $^{*}p<0.05$).}
    \label{fig:approx-full-spider}
\end{figure*}

\subsection{Number of Experts}\label{sec:app-k-experts}
\Cref{fig:k-experts} shows expected accuracy for within-model $\min$ ensembles as prompts are added one at a time. The step from one to two experts yields a gain in most settings. Beyond two experts, expected accuracy increases further for Llama and Phi, whereas it decreases for Qwen on BBH and SPIDER. 

\begin{figure*}
    \centering
    \includegraphics[width=\linewidth]{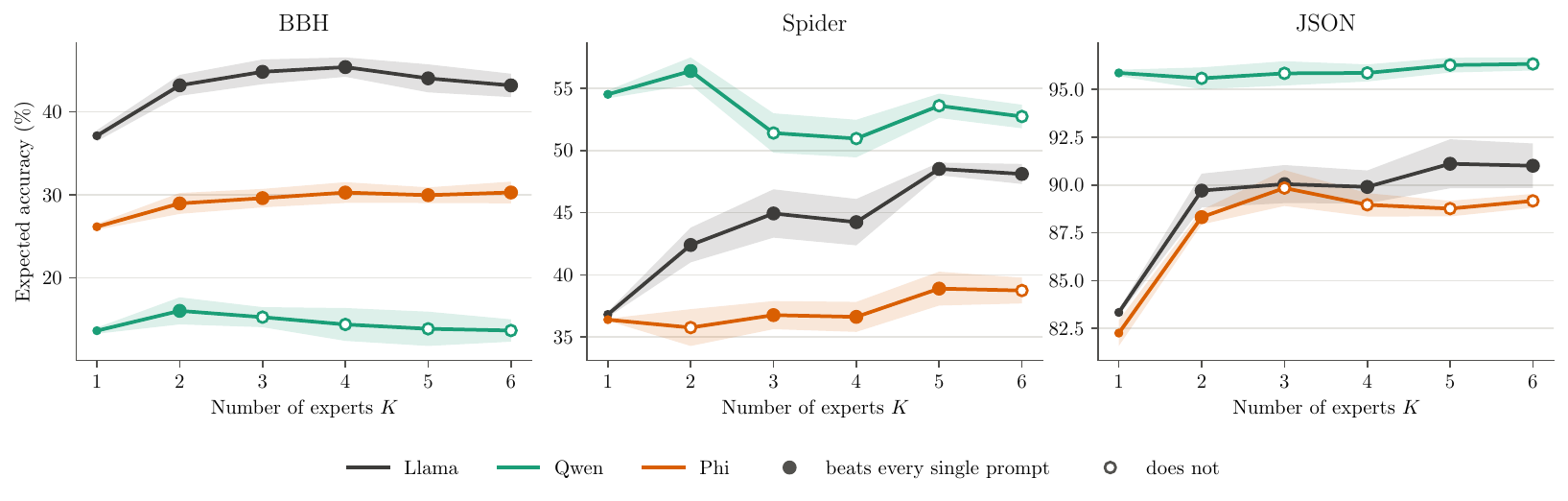}
    \caption{Within-model $\min$ ensembles (token-level SMC, $M{=}10$) with $K$ experts (prompts) and 95\% CIs across 5 seeds. $K{=}1$ is the first prompt alone, and each step adds one prompt. Filled markers: the ensemble beats every single prompt in its pool. Hollow markers: it does not.}
    \label{fig:k-experts}
\end{figure*}

\subsection{ESS Threshold}\label{sec:app-ess}

\Cref{tab:ess-min} ablates the resampling threshold $\essThreshold$ for within-model $\min$ ensembles with $M{=}10$ particles. Expected accuracy is largely insensitive to $\essThreshold$. Within each task--model pair, the means for different thresholds differ by at most 0.8 points (Spider for Phi), which lies within the 95\% CIs. This suggests, that at $M{=}10$, the improvements from token-level SMC over local ensembling in \cref{tab:main-results} come mainly from the importance weights, which correct the proposal towards the global target, and less from resampling.

\begin{table}[ht]
\centering
\small
\caption{ESS threshold $\eta$ (within-model, $\min$ operator, token SMC, $M{=}10$), expected accuracy, mean $\pm$ 95\% CI over 5 seeds. $\eta{=}0$ never resamples. \textbf{Bold}: largest value per row.}
\label{tab:ess-min}
\begin{tabular}{ll cccc}
\toprule
Task & Model & $\eta{=}0$ & $\eta{=}0.3$ & $\eta{=}0.6$ & $\eta{=}0.9$ \\
\midrule
BBH & Llama & \res{43.0}{1.6} & \res{43.0}{1.6} & \res{43.0}{1.6} & \win{\res{43.2}{1.3}} \\
BBH & Qwen & \res{15.5}{0.5} & \res{15.5}{0.6} & \res{15.6}{0.8} & \win{\res{16.0}{1.6}} \\
BBH & Phi & \win{\res{29.2}{1.3}} & \win{\res{29.2}{1.3}} & \win{\res{29.2}{1.3}} & \res{29.0}{1.3} \\
\midrule
SPIDER & Llama & \res{42.1}{0.8} & \res{42.2}{0.7} & \res{41.9}{1.1} & \win{\res{42.4}{1.4}} \\
SPIDER & Qwen & \res{56.2}{1.0} & \win{\res{56.6}{1.4}} & \res{56.5}{1.4} & \res{56.4}{1.1} \\
SPIDER & Phi & \res{36.2}{0.6} & \res{36.2}{0.6} & \win{\res{36.6}{1.6}} & \res{35.8}{1.5} \\
\midrule
JSON & Llama & \res{89.6}{0.7} & \win{\res{89.7}{0.8}} & \res{89.6}{0.7} & \win{\res{89.7}{0.9}} \\
JSON & Qwen & \win{\res{95.6}{0.5}} & \win{\res{95.6}{0.5}} & \win{\res{95.6}{0.5}} & \win{\res{95.6}{0.6}} \\
JSON & Phi & \win{\res{88.3}{0.5}} & \win{\res{88.3}{0.5}} & \win{\res{88.3}{0.5}} & \win{\res{88.3}{0.4}} \\
\bottomrule
\end{tabular}
\end{table}

\clearpage
\section{Prompts}\label{sec:prompts}
As discussed in \autoref{sec:experiments}, all of the models in our experiments are instruction-tuned. The detailed prompt templates for all tasks are given below.

\subsection{JSON Schema}
\begin{promptbox}{Prompt 1}
(*\promptrole{system}*)
You need to generate a JSON object that matches the schema below. Output the JSON object on a single line. DO NOT use multiple lines and DO NOT output any other text.

(*\promptrole{user}*)
<example JSON schema>

(*\promptrole{assistant}*)
<example JSON object>

(*\promptrole{user}*)
<target JSON schema>
\end{promptbox}

\begin{promptbox}{Prompt 2}
(*\promptrole{system}*)
Task: JSON object generation from schema specification.
Input: JSON Schema (draft-04 or later)
Output: Valid JSON object satisfying all schema constraints
Format: Single-line JSON, no extraneous text

(*\promptrole{user}*)
<example JSON schema>

(*\promptrole{assistant}*)
<example JSON object>

(*\promptrole{user}*)
<target JSON schema>
\end{promptbox}

\clearpage
\subsection{Big-Bench Hard (Word Sorting)}
\begin{promptbox}{Prompt 1}
(*\promptrole{system}*)
You are a helpful assistant that sorts words in alphabetical order. Your task is to take a list of words and return them sorted alphabetically. Alphabetical order means sorting words from A to Z, comparing character by character from left to right. For example, 'apple' comes before 'banana', and 'cat' comes before 'dog'. When words start with the same letter, compare the next letters. Output ONLY the sorted words separated by commas, with no additional text, explanation, or commentary. Do not include any prefixes, suffixes, or extra formatting - just the comma-separated sorted words.

(*\promptrole{user}*)
Sort the following words in alphabetical order.

Input: <list of words>
Output:
\end{promptbox}

\begin{promptbox}{Prompt 2}
(*\promptrole{system}*)
You are a helpful assistant that sorts words in alphabetical order. Your task is to take a list of words and return them sorted alphabetically. Output ONLY the sorted words separated by commas, with no additional text, explanation, or commentary.

(*\promptrole{user}*)
<example list of words> ->

(*\promptrole{assistant}*)
<example sorted words>

(*\promptrole{user}*)
<list of words> ->
\end{promptbox}

\clearpage
\subsection{SPIDER}
\begin{promptbox}{Prompt 1}
(*\promptrole{system}*)
You are a coding assistant helping an analyst answer questions over business data in SQL. More specifically, the analyst provides you a database schema (tables in the database along with their column names and types) and asks a question about the data that can be solved by issuing a SQL query to the database. In response, you write the SQL statement that answers the question. You do not provide any commentary or explanation of what the code does, just the SQL statement ending in a semicolon.
 Here is the database schema for all following queries:

<Columns=[]+FK schema string>

(*\promptrole{user}*)
Please write me a SQL statement that answers the following question: <example question>

Remember, DO NOT provide any commentary or explanation of what the code does, just the SQL statement ending in a semicolon.

(*\promptrole{assistant}*)
<example SQL query>;

(*\promptrole{user}*)
Please write me a SQL statement that answers the following question: <target question>

Remember, DO NOT provide any commentary or explanation of what the code does, just the SQL statement ending in a semicolon.
\end{promptbox}

\begin{promptbox}{Prompt 2: three fixed out-of-domain examples, each with its own schema}
(*\promptrole{system}*)
You are a coding assistant helping an analyst answer questions over business data in SQL. More specifically, the analyst provides you a database schema (tables in the database along with their column names and types) and asks a question about the data that can be solved by issuing a SQL query to the database. In response, you write the SQL statement that answers the question. You do not provide any commentary or explanation of what the code does, just the SQL statement ending in a semicolon.

(*\promptrole{user}*)
Here is a database schema:

<Table(Columns)+FK schema string of the example 1 database> Please write me a SQL statement that answers the following question: <example 1 question>

Remember, DO NOT provide any commentary or explanation of what the code does, just the SQL statement ending in a semicolon.

(*\promptrole{assistant}*)
<example 1 SQL query>;

[... the same user and assistant turns for examples 2 and 3 ...]

(*\promptrole{user}*)
Here is a database schema:

<Table(Columns)+FK schema string of the target database> Please write me a SQL statement that answers the following question: <target question>

Remember, DO NOT provide any commentary or explanation of what the code does, just the SQL statement ending in a semicolon.
\end{promptbox}

\subsection{GSM8K}
\begin{promptbox}{Prompt 1: few-shot CoT}
(*\promptrole{system}*)
You are a helpful assistant that solves grade school math word problems. Show your reasoning step by step, then state the final answer.

(*\promptrole{user}*)
Q: There are 15 trees in the grove. Grove workers will plant trees in the grove today. After they are done, there will be 21 trees. How many trees did the grove workers plant today?
A: There are 15 trees originally. Then there were 21 trees after some more were planted. So there must have been 21 - 15 = 6. The answer is 6.
Q: If there are 3 cars in the parking lot and 2 more cars arrive, how many cars are in the parking lot?
A: There are originally 3 cars. 2 more cars arrive. 3 + 2 = 5. The answer is 5.
[... 3 more few-shot examples ...]
Q: <GSM8K question>
A:
\end{promptbox}

\begin{promptbox}{Prompt 2: CoT with a precision instruction}
(*\promptrole{system}*)
You are a precise mathematical reasoner. Your goal is correctness, not speed.

For each problem, follow these steps:
1. Extract every numerical fact stated in the problem,  nothing more, nothing less.
2. State exactly what quantity is being asked for.
3. Determine the sequence of operations needed, giving a reason for each.
4. Carry out each arithmetic step explicitly.
5. State your answer as "The answer is X."

(*\promptrole{user}*)
Problem: <GSM8K question>
\end{promptbox}

\clearpage
\subsection{MMLU}
\begin{promptbox}{Prompt 1: text with answers, format 1}
(*\promptrole{system}*)
You are a helpful assistant that answers multiple choice questions. Respond with ONLY the exact text of the correct answer choice. Do not include any explanation or additional text.

(*\promptrole{user}*)
The following are multiple choice questions (WITH answers) about <subject name>'

<few-shot question 1>
Answers'
A. <few-shot 1 choice A>
B. <few-shot 1 choice B>
C. <few-shot 1 choice C>
D. <few-shot 1 choice D>
Answer'
<few-shot 1 answer text>

[... one block per few-shot example, 5 in total ...]

<MMLU question>
Answers'
A. <choice A>
B. <choice B>
C. <choice C>
D. <choice D>
Answer'
\end{promptbox}

\begin{promptbox}{Prompt 2: text with answers, format 2}
(*\promptrole{system}*)
You are a helpful assistant that answers multiple choice questions. Respond with ONLY the exact text of the correct answer choice. Do not include any explanation or additional text.

(*\promptrole{user}*)
The following are multiple choice questions (With answers) about <subject name>(*\promptlit{sep}*)

<few-shot question 1>(*\promptlit{sep}*)
Answers:
A. <few-shot 1 choice A>(*\promptlit{sep}*)
B. <few-shot 1 choice B>(*\promptlit{sep}*)
C. <few-shot 1 choice C>(*\promptlit{sep}*)
D. <few-shot 1 choice D>(*\promptlit{sep}*)
Answer:
<few-shot 1 answer text>

[... one block per few-shot example, 5 in total ...]

<MMLU question>(*\promptlit{sep}*)
Answers:
A. <choice A>(*\promptlit{sep}*)
B. <choice B>(*\promptlit{sep}*)
C. <choice C>(*\promptlit{sep}*)
D. <choice D>(*\promptlit{sep}*)
Answer:
\end{promptbox}

\newpage
\clearpage

\end{document}